\documentclass[10pt]{article}
\usepackage[margin=1in]{geometry} 
\usepackage[english]{babel}
\usepackage[utf8x]{inputenc}
\usepackage[T1]{fontenc}

\usepackage[dvipsnames]{xcolor}
\usepackage{graphicx}
\usepackage{caption}
\usepackage{subcaption}
\usepackage{enumitem}
\usepackage{wrapfig}
\usepackage{algorithm}
\usepackage[noend]{algpseudocode}
\usepackage{amsmath}
\usepackage{amssymb}
\usepackage{amsthm}
\usepackage[title]{appendix}
\usepackage{titletoc}
\usepackage[square, authoryear]{natbib}
\usepackage[
    colorlinks=true, 
    linkcolor=blue!70!black, 
    citecolor=blue!70!black,
    urlcolor=black,
    breaklinks=true]{hyperref} 
\usepackage{authblk}

\newtheorem{theorem}{Theorem}

\newtheorem{assumption}{Assumption}
\newtheorem{proposition}{Proposition}
\newtheorem{lemma}{Lemma}
\newtheorem{remark}{Remark}

\DeclareMathOperator{\E}{E}

\newcommand{\lambdat}{\tilde{\lambda}}
\newcommand{\logn}{\log n}

\newcommand{\mbP}{\mathbb{P}}
\newcommand{\mbR}{\mathbb{R}}

\newcommand{\mcA}{\mathcal{A}}

\newcommand{\mcD}{\mathcal{D}}
\newcommand{\mcE}{\mathcal{E}}
\newcommand{\mcF}{\mathcal{F}}

\newcommand{\mcM}{\mathcal{M}}
\newcommand{\mcN}{\mathcal{N}}

\newcommand{\mcH}{\mathcal{H}}

\newcommand{\mcT}{\mathcal{T}}
\newcommand{\mcW}{\mathcal{W}}
\newcommand{\mcS}{\mathcal{S}}
\newcommand{\mcX}{\mathcal{X}}
\newcommand{\mcY}{\mathcal{Y}}

\title{Smoothness Adaptive Hypothesis Transfer Learning}

\author[1]{Haotian Lin}
\author[1,2]{Matthew Reimherr}

\affil[1]{Department of Statistics, The Pennsylvania State University}
\affil[2]{Amazon Science}
\date{}

\begin{document}
\maketitle

\begin{abstract}
Many existing two-phase kernel-based hypothesis transfer learning algorithms employ the same kernel regularization across phases and rely on the known smoothness of functions to obtain optimality. Therefore, they fail to adapt to the varying and unknown smoothness between the target/source and their offset in practice. In this paper, we address these problems by proposing \textit{\underline{S}moothness \underline{A}daptive \underline{T}ransfer \underline{L}earning} (SATL), a two-phase kernel ridge regression(KRR)-based algorithm. We first prove that employing the misspecified fixed bandwidth Gaussian kernel in target-only KRR learning can achieve minimax optimality and derive an adaptive procedure to the unknown Sobolev smoothness. Leveraging these results, SATL employs Gaussian kernels in both phases so that the estimators can adapt to the unknown smoothness of the target/source and their offset function. We derive the minimax lower bound of the learning problem in excess risk and show that SATL enjoys a matching upper bound up to a logarithmic factor. The minimax convergence rate sheds light on the factors influencing transfer dynamics and demonstrates the superiority of SATL compared to non-transfer learning settings.
While our main objective is a theoretical analysis, we also conduct several experiments to confirm our results.
\end{abstract}

\section{Introduction}\label{sec: inrtroduction}
Nonparametric regression is one of the most prevalent statistical problems studied in many communities in past decades due to its flexibility in modeling data. A large number of algorithms have been proposed like kernel regression, local regression, smoothing splines, and regression trees, to name only a few. However, the effectiveness of all the algorithms in these existing works is based on having sufficient samples drawn from the same target domain. When samples are scarce, either due to costs or other constraints, the performance of these algorithms can suffer empirically and theoretically. 

Hypothesis Transfer Learning (HTL) \citep{li2007bayesian,kuzborskij2013stability,du2017hypothesis}, which leverages models trained on the source domain and uses samples from the target domain to learn the model shift to the target model, is an appealing and promising mechanism. When the parameters of interest are infinite dimensional, \citet{lin2022transfer} employed the reproducing kernel Hilbert space (RKHS) distance as a metric for assessing similarity in functional regression frameworks, linking the transferred knowledge to the employed RKHS structure. They leveraged offset transfer learning (OTL),  which is one of the most popular HTL algorithms, to obtain target estimators in a two-phase manner, i.e. a trained source model is first learned on the large sample size source dataset and the offset model between the target and source is then estimated via target dataset and trained source models. However, a noteworthy observation is that both phases of estimating the source model and offset model utilize the same RKHS regularization. This raises questions regarding the adherence to the essence of transfer learning, where the offset should ideally possess a simpler structure than the target and source model to reflect the similarity. A similar limitation also appears in a series of two-phase HTL algorithms for finite-dimensional models (e.g. multivariate/high-dimensional linear regression) \citep{bastani2021predicting,li2022transfer}, which typically utilize the $\ell^{1}$ or $\ell^{2}$ norm of the offset parameters as the similarity measure. While $\ell^{1}$-norm can reflect sparsity and $\ell^{2}$ usually serves to control complexity, using the same norm in both phases is more defensible in finite-dimensional models than in infinite-dimensional ones (e.g., nonparametric models). 

In the realm of nonparametric regression, although OTL has shown great success in practice, there are only a few studies that provide theoretical analysis \citep{wang2015generalization,du2017hypothesis}, and these works can still be limited in terms of problem settings, estimation procedures, and theoretical bounds. For example, although \citet{wang2015generalization} noticed the nature of simple offsets, they didn't use any quantity (like Sobolev or H\"older smoothness) to formularize the difference of target/source models and their offset. Their KRR-based OTL algorithm also employed the same kernel to train the source model and the offset and thus has similar limitations as \citet{lin2022transfer} methodologically. \citet{du2017hypothesis} formalized the varying structures via different H\"older smoothness, but the theoretical results derived are under too ideal assumptions, unverifiable in practice. Besides, neither their approaches nor the statistical convergence rates were adaptive and their upper bound overlooked certain factors, failing to provide deeper insights into the influence of domain properties on transfer learning efficacy. This raises the following fundamental question that motivates our study: 
\begin{center}
    \textit{
    Can we develop an HTL algorithm so that the different structures (smoothness) of the target/source functions and their offset can be adaptively learned?  
    }
\end{center}

\paragraph{Main contributions.} This work answers the above question positively and makes the following contributions:

We propose \textit{\underline{S}moothness \underline{A}daptive \underline{T}ransfer \underline{L}earning} (SATL), building upon the prevalent two-phase offset transfer learning paradigm. Specifically, we study the setting where the target/source function lies in Sobolev space with order $m_0$ while the offset function lies in Sobolev space with order $m$ (where $m>m_{0}$). One key feature of SATL is its ability to adapt to the unknown and varying smoothness of the target, source, and offset functions.

We first begin by establishing the robustness of the Gaussian kernel in misspecified KRR, i.e. for regression functions belonging to certain fractional Sobolev spaces (or RKHSs that are norm equivalent to such Sobolev spaces), employing a fixed bandwidth Gaussian kernel in target-only KRR yields minimax optimal generalization error. Remarkably, the optimal order of the regularization parameters follows an exponential pattern, which differs from the variable bandwidth setting and we conduct comprehensive experiments to support the finding. Furthermore, we demonstrate that an estimator, developed through standard training and validation methods, achieves the same optimality up to a logarithmic factor (often called the price of adaptivity), without prior knowledge of the function's true smoothness.

Leveraging these new results of the Gaussian kernel, SATL employs Gaussian kernels in both learning phases, ensuring its adaptability to the diverse and unknown smoothness levels $m_{0}$ and $m$. We also establish the minimax statistical lower bound for the learning problem in terms of excess risk and show that SATL achieves minimax optimality since it enjoys a matching upper bound (up to logarithmic factors). Crucially, our results shed light on the impact of signal strength from both domains on the efficacy of OTL, which, to the best of our knowledge, has been largely overlooked in the existing literature. This insight enhances our understanding of the contributions of each phase in the transfer learning process.

\subsection{Related Literature}\label{subsec: literature review}

\textbf{Transfer Learning}: 
OTL (a.k.a. biases regularization transfer learning) has been extensively researched in supervised regression. The work in \citet{kuzborskij2013stability,kuzborskij2017fast} focused on OTL in linear regression, establishing generalization bounds through Rademacher complexity. \citet{wang2015generalization} derived generalization bounds for applying KRR on OTL, without formularizing the simple offset structure. \citet{wang2016nonparametric} assumed target/source regression functions in the Sobolev ellipsoid (a subspace of Sobolev space) with order $m_{0}$ and the offset in a smoother power Sobolev ellipsoid. They used finite orthonormal basis functions for modeling, which becomes restrictive if the chosen basis is misaligned with the eigenfunctions of the Sobolev ellipsoid. \citet{du2017hypothesis} further proposed a transformation function for the offset, thereby integrating many preview OTL studies and offering upper bounds on excess risk for both kernel smoothing and KRR. Apart from regression settings, generalization bounds for classification problems with surrogate losses have been studied in \citet{aghbalou2023hypothesis} via stability analysis techniques. Other results that study HTL outside OTL can be found in \citet{li2007bayesian,orabona2009model,cheng2015joint}. Besides, OTL can also be viewed as a case of representation learning \citet{du2020few,tripuraneni2020theory,xu2021representation} by viewing the trained source model as a representation for target tasks.

The idea of OTL has also been adopted by the statistics community recently, which typically involves regularizing the offset via different metrics in parameter spaces.
For example, \citet{bastani2021predicting,li2022transfer,tian2022transfer} considered $\ell^{1}$-distance OTL for high-dimensional (generalized) linear regression. \citet{duan2022adaptive,tian2023learning} considered $\ell^{2}$-distance for multi-tasking learning. \citet{lin2022transfer} utilized RKHS-distance to measure the similarity between functional linear models from different domains. However, all of these works used the same type of distance while estimating the source model and the offset. In the nonparametric regression context, another study by \citet{cai2022transfer} assumed the target/source models lie in H\"older spaces while the offset function can be approximated with any desired accuracy by a polynomial function in terms of $L_{1}$ distance. They proposed a confidence thresholding estimator based on local polynomial regression. Although with adaptive properties, their proposed approach can be computationally intensive in practice.
Kernel ridge regression under the covariate shift setting is also studied in several works. \citet{ma2022optimally} proposed an estimation scheme by reweighting the loss function using the likelihood ratio between the target and source domains. Additionally, \citet{wang2023pseudo} introduced a pseudo-labeling algorithm to address TL scenarios where the labels in the target domain are unobserved.

\textbf{Misspecification in KRR}: 
This line of research focuses on using misspecified kernels in target-only KRR to achieve optimal statistical convergence rates.

In the realm of variable bandwidths, \citet{eberts2013optimal} conducted a study on the convergence rates of excess risk for KRR using a Gaussian kernel when the true regression function lies in a Sobolev space. They found that appropriate choices of regularization parameters and the bandwidth will yield a non-adaptive rate that can be arbitrarily close to the optimal rate under the bounded response assumption. Building upon this work, \citet{hamm2021adaptive} further improved the non-adaptive rate, achieving optimality up to logarithmic factors. It is worth noting that their results show that the optimal parameters should be tuned in a polynomial pattern, which is different from ours. Apart from regression setting, \citet{li2019optimality} studied using variable bandwidth Gaussian kernels to achieve optimality in a series of nonparametric statistical tests.

Another line of research considers fixed bandwidth kernels. For instance, \citet{wang2022gaussian} investigated the misspecification of Mat\'ern kernel-based KRR. They demonstrated that even when the true regression functions belong to a Sobolev space, utilizing misspecified Mat\'ern kernels can still yield an optimal convergence rate or, in some cases, a slower convergence rate (referred to as the saturation effect of KRR). Similarly, several other works \citep{steinwart2009optimal, dicker2017kernel, fischer2020sobolev} have presented similar results, but with assumptions directly imposed on the eigenvalues and eigenfunctions of the kernel, rather than the associated RKHSs.

\section{Preliminaries}

\paragraph{Problem Formulation.}
Consider the two nonparametric regression models
\begin{equation*}
    y_{p,i} = f_{p}( x_{p,i} ) + \epsilon_{p,i}, \quad p\in\{T,S\}
\end{equation*}
where $p$ is the task index ($T$ for target and $S$ for source), $f_{p}$ are unknown regression functions, $x_{p,i} \in \mcX\subset \mathbb{R}^{d}$ is a compact set with positive Lebesgure measure and Lipschitz boundary, and $\epsilon_{p,i}$ are i.i.d. random noise with zero mean. The target and source regression function, $f_{T}$ and $f_{S}$, belongs to the (fractional) Sobolev space $H^{m_{0}}$ with smoothness $m_0 \geq d/2$ over $\mcX$. The joint probability distribution $\rho_{p}(x,y)$ is defined on $\mcX \times \mcY$ for the data points $\{(x_{p,i}, y_{p,i})\}_{i=1}^{n_{p}}$, and $\mu_{p}(x)$ represents the marginal distribution of $\rho_{p}$ on $\mcX$. In this work, we assume the posterior drift setting, where $\mu_{T}(x)$ is equal to $\mu_{S}(x)$, while the regression function $f_{T}$ differs from $f_{S}$. The goal of this paper is to find a function $\hat{f}_{T}$ based on the combined data $\{(x_{T,i}, y_{T,i})\}_{i=1}^{n_{T}} \cup \{(x_{S,i}, y_{S,i})\}_{i=1}^{n_{S}}$ that minimizes the generalization error on the target domain, i.e. 
\begin{equation*}
    \mcE(\hat{f}_{T}) = \E_{x\sim \mu_{T}}[ ( \hat{f}_{T}(x) - f_{T}(x) )^{2} ].
\end{equation*}

\paragraph{Non-Transfer Scenario.}
In the absence of source data, recovering $f_{T}$ using KRR is referred to as target-only learning and has been extensively studied. We now state some of its well-known properties.
\begin{proposition}[Target-only Learning] \label{proposition: target-only learning}
    For a symmetric and positive semi-definite kernel $K:\mcX \times \mcX \rightarrow \mathbb{R}$, let $\mcH_{K}$ be the RKHS associated with $K$ \citep{wendland2004scattered}. The KRR estimator is 
\begin{equation*}
    \hat{f}_{T} = \underset{f\in \mcH_{K}}{\operatorname{argmin}} \left \{ \frac{1}{n_{T}} \sum_{i=1}^{n_{T}} (y_{T,i} - f(x_{T,i}) )^2 + \lambda \|f\|_{\mcH_{K}}^2 \right\}, 
\end{equation*}
and we call the kernel $K$ as the imposed kernel.
Then the convergence rate of the generalization error of $\hat{f}_{T}$, $\mcE(\hat{f}_{T})$ is given as follows.
\begin{enumerate}
    \item (Well-specified Kernel) If $\mcH_{K}$ coincides with $H^{m_{0}}$, $\mcE(\hat{f}_{T})$ can reach the standard minimax convergence rate in high-probability given $\lambda \asymp n^{- \frac{2m_{0}}{2m_{0} + d } }$, i.e.
    \begin{equation*}
        \mcE(\hat{f}_{T}) = O_{\mathbb{P}}\left( n_{T}^{-\frac{2m_{0}}{2m_{0} + d}} \right).
    \end{equation*}

    \item (Misspecified Kernel) If the $K$ is  the Mat\'ern kernel then its induced space is isomorphic to $H^{m_{0}'}$ with $m_{0}' > \frac{d}{2}$.  Furthermore, given $\lambda \asymp n^{- \frac{2m_{0}'}{2m_{0} + d} }$ and $\gamma = \min\{2,\frac{m_{0}}{m_{0}'}\}$ , then 
    \begin{equation*}
        \mcE(\hat{f}_{T}) = O_{\mathbb{P}} \left( n_{T}^{-\frac{2\gamma m_{0}'}{2\gamma m_{0}' + d}} \right).
    \end{equation*}

    \item (Saturation Effect) For $m_{0}' < \frac{m_{0}}{2}$ and any choice of parameter $\lambda(n_{T})$ satisfying that $n_{T}\rightarrow\infty$, we have 
    \begin{equation*}
        \mcE(\hat{f}_{T}) = \Omega_{\mathbb{P}} \left( n_{T}^{-\frac{4 m_{0}'}{4m_{0}' + d}} \right).
    \end{equation*}
\end{enumerate}
\end{proposition}
The well-specified result is well-known and can be found in a line of past work \citep{geer2000empirical,caponnetto2007optimal}. The misspecified kernel result comes from a combination (with a modification) of Theorem 15 and 16 in \citet{wang2022gaussian}. The saturation effect is proved by \citet{li2023saturation}. The Proposition~\ref{proposition: target-only learning} tells the fact that for target-only KRR, even when the smoothness of the imposed RKHS, $m_{0}'$, disagrees with the smoothness $m_{0}$ of the Sobolev space to which $f_{T}$ belongs, the optimal rate of convergence is still achievable if $m_{0}' \geq m_{0}/2$ with the $\lambda$ appropriately chosen. However, if $m_{0}' < m_{0}/2$, i.e. the true function is much smoother than the estimator itself, the saturation effect occurs, meaning that the information theoretical lower bound $n_{T}^{-2m_{0}/(2m_{0}+d)}$ seemingly cannot be achieved regardless of the tuning of the regularization parameters in KRR \citep{bauer2007regularization,gao2008knowledge}.

\paragraph{Transfer Learning Framework.}
We introduce the nonparametric version of OTL, which serves as the backbone of our proposed algorithm. Formally, OTL obtains the estimator for $f_{T}$ as $\hat{f}_{T} = \hat{f}_{S} + \hat{f}_{\delta}$ via two phases. In the first phase, it obtains $\hat{f}_{S}$ by KRR with the source dataset $\{(x_{S,i}, y_{S,i})\}_{i=1}^{n_{S}}$. In the second phase, it generates pseudo offset labels $\{ y_{T,i} - \hat{f}_{S}(x_{T, i}) \}_{i=1}^{n_{T}}$ and then learns the $\hat{f}_{\delta}$ via KRR by replacing target labels by pseudo offset labels. The main idea of OTL is that the $f_{S}$ can be learned well given sufficiently large source samples and the offset $f_{\delta}$ can be learned with much fewer target samples. We formulate the OTL variant of KRR as Algorithm~\ref{algo: Two-step TL KRR}.
\begin{algorithm}[ht]
\caption{OTL-KRR}\label{algo: Two-step TL KRR}
    
    \textbf{Input:} Target and source training data
    $ \{ (x_{T,i}, y_{T,i}) \}_{i=1}^{n_{T}} \cup  \{ (x_{S,i}, y_{S,i}) \}_{i=1}^{n_{S}} \}$; Self-specified KRR imposed kernel $K$ \\
    \textbf{Output:} Target function estimator $\hat{f}_{T} = \hat{f}_{S} + \hat{f}_{\delta}$.

    \underline{\text{Phase 1:}}
    \begin{equation*}
        \hat{f}_{S} = \underset{f\in \mcH_{K}}{\operatorname{argmin}}  \frac{1}{n_{S}} \sum_{i=1}^{n_{S}} (y_{S,i} - f(x_{S,i}) )^2 + \lambda_{1} \|f\|_{\mcH_{K}}^2
    \end{equation*}

    \underline{\text{Phase 2:}} 
    \begin{equation*}
        \hat{f}_{\delta} = \underset{f\in \mcH_{K}}{\operatorname{argmin}}  \frac{1}{n_{T}} \sum_{i=1}^{n_{T}} (y_{T,i} - \hat{f}_{S}(x_{T,i}) - f(x_{T,i}) )^2 + \lambda_{2} \|f\|_{\mcH_{K}}^2 
    \end{equation*}
    
\end{algorithm}

\paragraph{Model Assumptions.}
We first state the smoothness assumption on the offset function $f_{\delta}:= f_{T} - f_{S}$.

The learning framework (Algorithm~\ref{algo: Two-step TL KRR}) reveals a smoothness-agnostic nature: the imposed kernels $K$ (also the associated RKHSs) stay the same regardless of the level of smoothness of $f_{T},f_{S}$ and $f_{\delta}$. More specifically, based on Proposition~\ref{proposition: target-only learning}, the learning algorithm is rate optimal when the smoothness of both imposed RKHSs $\mcH_{K}$ in both steps matches that of $f_{S}$, and $f_{\delta}$, i.e. the smoothness of $f_{S}$ and $f_{\delta}$ stay the same. However, in the transfer learning context, such a smoothness condition on the offset function may not be precise enough. One should rather consider the offset function smoother than the target/source functions themselves to represent the similarity between $f_{S}$ and $f_{T}$.

To illustrate this point, consider the following example. Suppose $f_{T} = f_{S} + f_{\delta}$ where $f_{S}$ is a complex function with low smoothness (less regularized) while $f_{\delta}$ is rather simple (well regularized), e.g. a linear function. Then $f_{S}$ can be estimated well via larger $n_{S}$ while $f_{\delta}$ is a highly smooth function and can also be estimated well via small $n_{T}$ due to its simplicity. In this example, the effectiveness of OTL relies on the similarity between $f_{T}$ and $f_{S}$, i.e. the offset $f_{\delta}$ possessing a ``simpler'' structure than the target and source models. Such ``simpler'' offset assumptions have been proven to make OTL effective in high-dimensional regression literature \citet{bastani2021predicting,li2022transfer,tian2022transfer}, where authors assume the offset coefficient should be sparser than target/source coefficients. This motivates our endeavor to introduce the following smoothness assumptions to quantify the similarity between target and source domains. 
\begin{assumption}[Smoothness of Target/Source] \label{assump1}
     There exists an $m_0 \geq d/2$ such that $f_{T}$ and $f_{S}$ belong to $H^{m_0}$.
\end{assumption}
\begin{assumption}[Smoothness of Offset] \label{assump2} 
There exists an $m\geq m_{0}$ such that $f_{\delta} := f_{T} - f_{S}$ belongs to $H^{m}$.    
\end{assumption}
\begin{remark}
The results of this paper are applicable not only to Sobolev spaces but also to those general RKHSs that are norm equivalent to Sobolev spaces. Thus we can assume that $f_{S}$, $f_{T}$, and $f_{\delta}$ belong to RKHSs with different regularity. Due to norm equivalency, our discussion is primarily focused on Sobolev spaces and we refer readers to Appendix~\ref{apd: Norm Equivalency between RKHS and Sobolev Space} for the analysis pertaining to general RKHSs.
\end{remark}

Assumption~\ref{assump1} is a very common assumption in nonparametric regression literature and Assumption~\ref{assump2} naturally holds if Assumption~\ref{assump1} is satisfied. Compared to the smooth offset assumption in \citet{wang2016nonparametric} where $f_{\delta}$ is assumed to belong to the power space of $H^{m_{0}}$, our setting presents a unique challenge. Since we consider the offset function in a Sobolev space with higher smoothness, which doesn't necessarily share the same eigenfunctions with $H^{m_{0}}$, this renders orthonormal basis modeling less promising. Assumption~\ref{assump2} also makes our setting conceptually align with contemporary transfer learning models. For instance, in prevalent pretraining-finetuning neural networks, the pre-trained feature extractor tends to encompass a greater number of layers, while the newly added fine-tuning structure typically involves only a few layers. 
In this analogy, $m_0$ and $m$ in our setting are akin to the deeper pre-trained layers and the shallow fine-tuned layers.

We also state a standard assumption that frequently appears in KRR literature \citep{fischer2020sobolev,zhang2023optimality} to establish theoretical results, which controls the noise tail probability decay speed.
\begin{assumption}[Moment of error] \label{assump: error tail} 
There are constants $\sigma$,$L>0$ such that for any $r\geq 2$, the noise, $\epsilon$, satisfies 
\begin{equation*}
    \E \left[ |\epsilon_{p}|^{r} \mid x  \right] 
    \leq \frac{1}{2}r! \sigma^{2} L^{r-2}, \quad \text{for} \quad p\in\{T,S\}.
\end{equation*}
\end{assumption}


\section{Target-Only KRR with Gaussian kernels}\label{sec: target-KRR}
To achieve optimality in Algorithm~\ref{algo: Two-step TL KRR} under the smoothness assumptions, an applicable approach is to impose distinct kernels that can accurately capture the correct smoothness of $f_{S}$ and $f_{\delta}$ during both phases. This approach, however, faces the practical challenge of identifying the unknown smoothness $m_0$ and $m$, which, in turn, induce different kernels and RKHS; this is an issue also prevalent in the target-only KRR context.

One potential solution is to leverage the robustness of Mat\'ern kernels, i.e. employ a misspecified Mat\'ern kernel as the imposed kernel in KRR. As indicated by Proposition~\ref{proposition: target-only learning}, the optimal convergence rate is still attainable for some appropriately chosen Mat\'ern kernels and regularization parameters. Nonetheless, this still faces two problems: 
\begin{enumerate}
    \item [(1)] The rate with misspecified Mat\'ern kernel in Proposition~\ref{proposition: target-only learning} is still not adaptive, i.e. one still needs to know the true smoothness when selecting $\lambda_{1}$ and $\lambda_{2}$.

    \item [(2)] The risk of the saturation effect happening in KRR when a less smooth kernel is chosen.
\end{enumerate}
While the former can be potentially addressed by cross-validation or data-driven adaptive approach, the second one is more fatal as one might end up choosing a less smooth kernel and never be able to achieve the information theoretical lower bounds because of the saturation effect.

Hence, there is a clear demand for a kernel with a more general robust property, i.e. in the target-only KRR context, for the regression function that lies in $H^{\alpha_{0}}$ with any $\alpha_{0}\geq d/2$, employing such a kernel ensures that there's always an optimal $\lambda$ such that the optimal convergence rate is achievable. Motivated by the fact that the Gaussian kernel is the limit of Mat\'ern kernel $K_{\nu}$ as $\nu \rightarrow \infty$ and the RKHS associated with the Gaussian kernel is contained in the Sobolev space $H^{\nu}$ for any $\nu > d/2$ \citep{fasshauer2011reproducing}, we show that the Gaussian kernel indeed possesses the desired property.

Consider the Target-Only learning KRR setting, where $f_{0}\in H^{\alpha_{0}}$ (we use $\alpha_{0}$ to denote smoothness in target-only context to distinguish from TL context) and the underlying supervised learning model setup as 
\begin{equation*}
    y_{i} = f_{0}(x_{i}) + \epsilon_{i}, \quad i = 1,\cdots,n. 
\end{equation*}

We first show the non-adaptive rate of the Gaussian kernel.
\begin{theorem}[Non-Adaptive Rate]\label{thm: non-adaptive rate of KRR with Gaussian}
    Under the Assumptions~\ref{assump: error tail}, let the imposed kernel, $K$, be the Gaussian kernel with fixed bandwidth and $\hat{f}$ be the corresponding KRR estimator based on data $\{(x_i,y_{i})\}_{i=1}^{n}$. If $f_{0}\in H^{\alpha_{0}}$, by choosing $log(1/\lambda) \asymp n^{\frac{2}{2\alpha_{0}+d}}$, for any $\delta \in (0,1)$, when $n$ is sufficient large, with probability at least $1-\delta$, we have
    \begin{equation*}
        \| \hat{f} - f_{0} \|_{L_{2}}^{2} \leq C \left(\log \frac{4}{\delta} \right)^2  n^{-\frac{2\alpha_{0}}{2\alpha_{0} + d}},
    \end{equation*}
    where $C$ is a universal constant independent of $n$ and $\delta$. 
\end{theorem}
\begin{remark}
    Although \citet{eberts2013optimal} has studied the robustness of Gaussian kernel on misspecified KRR,
    their results are built on variable bandwidths and the convergence rate can only be arbitrarily close to but not exactly reach the minimax optimal rate, given both bandwidth and $\lambda$ decay polynomially in $n$. In contrast, our result is built on fixed bandwidth Gaussian kernels and achieves the optimal rate with the optimal $\lambda$ that behaves differently from theirs.
\end{remark}

We note to the reader that while the RKHS associated with the Mat\'en kernel coincides with a Sobolev space (i.e. they are the same space with slightly different, though equivalent, norms), the Gaussian kernel does not, making the behavior of the optimal $\lambda$ totally different compared to the misspecified Mat\'ern kernel scenarios in Proposition~\ref{proposition: target-only learning}. Particularly, even if the Gaussian kernel is the limit of Mat\'ern kernel $K_{\nu}$ as $\nu\rightarrow\infty$, setting $m_{0}'$ as infinity in misspecified kernel case of Proposition~\ref{proposition: target-only learning} will never yield analytical results but only tells the optimal order of $\lambda$ should converge to $0$ faster than polynomial ($\lim_{m_{0}' \rightarrow \infty} n^{-2m_{0}'/(2m_{0}+d) } = 0$). On the other side, our result identifies $\lambda$ should converge to $0$ exponentially in $n$. To highlight our findings, we compare our results with existing state-of-the-art works on misspecified KRR and refer readers to Appendix~\ref{apd: target-only literature compare} for details.

To develop an adaptive procedure without known $\alpha_{0}$, we employ a standard training/validation approach \citep{steinwart2008support}. To this end, let $\mcA = \{\alpha_{\min}<\cdots<\alpha_{\max}\}$ with $\alpha_{\min}>d/2$ and $\alpha_{\max}$ large enough such that $\alpha_{0}\in\mcA$. Split dataset $\mcD = \{(x_{i},y_{i})\}_{i=1}^{n}$ into 
\begin{equation*}
\begin{aligned}
    &\mcD_{1} := \{(x_{1},y_{1},\cdots,(x_{j},y_{j}))\}\\
    &\mcD_{2} := \{(x_{j+1},y_{j+1},\cdots,(x_{n},y_{n}))\}
\end{aligned}
\end{equation*}
The adaptive estimator is obtained by following the training and validation approach.
\begin{enumerate}
    \item For each $\alpha \in \mcA$, obtain non-adaptive estimator $\hat{f}_{\lambda_{\alpha}}$ by KRR with dataset $\mcD_{1}$ and $\lambda$ setting to optimal order in Theorem~\ref{thm: non-adaptive rate of KRR with Gaussian}.
    \item Obtain the adaptive estimator $\hat{f}_{\lambda_{\hat{\alpha}}}$ by minimizing empirical $L_{2}$ error on $\mcD_{2}$, i.e. 
    \begin{equation*}
        \hat{f}_{\lambda_{\hat{\alpha}}} = \underset{\alpha \in \mcA}{\operatorname{argmin}} \left\{ \frac{1}{n-j} \sum_{i=j+1}^{n} (y_{i} - \hat{f}_{\lambda_{\alpha}}(x_{i}))^2 \right\}.
    \end{equation*}
\end{enumerate}
When constructing the collection of non-adaptive estimators over $\mcA$, Theorem~\ref{thm: non-adaptive rate of KRR with Gaussian} suggests choosing the regularization parameter $\lambda = \exp\{ -Cn^{2/2\alpha + d} \}$ for some constant $C$. In practice, this $C$ can be selected by cross-validation.

The following theorem shows the estimator from the training/validation approach achieves an optimal minimax rate up to a logarithm factor in $n$.
\begin{theorem}[Adaptive Rate]\label{thm: adaptive rate of single-task GKRR}
    Under the same conditions of Theorem~\ref{thm: non-adaptive rate of KRR with Gaussian} and $\mcA = \{\alpha_{1},\cdots,\alpha_{N}\}$ with $\alpha_{j} - \alpha_{j-1} \asymp 1/\logn$. Then, for $\delta \in (0,1)$, when $n$ is sufficient large, with probability $1-\delta$, we have 
    \begin{equation*}
        \mcE( \hat{f}_{\lambda_{\hat{\alpha}}} ) \leq C \left( \log\frac{4}{\delta} \right)^2 \left(\frac{n}{\logn} \right)^{-\frac{2\alpha_{0}}{2\alpha_{0} +d}},
    \end{equation*}
    where $C$ is a universal constant independent of $n$ and $\delta$. 
\end{theorem}
On the other hand, if the marginal distribution of $x$, $\mu$, is known, we show that Lepski's method \citep{lepskii1991problem} can also be used to obtain adaptive estimator without knowing $\alpha_{0}$, and it also achieves optimal nonadaptive rate up to a logarithm factor as training and validation approach does. We refer readers to Appendix~\ref{apd: Lepski's Method} for a detailed description of Lepski's method and its adaptive rate.

\section{Smoothness Adaptive Transfer Learning}\label{sec: SATL}
We formally propose Smoothness Adaptive Transfer Learning in Algorithm~\ref{algo: SATL}. 
\begin{algorithm}[ht]
\caption{\textit{\underline{S}moothness \underline{A}daptive \underline{T}ransfer \underline{L}earning} (SATL)}\label{algo: SATL}
    

    \textbf{Input:} Target and source dataset
    $\mcD_{T}$ and $\mcD_{S}$; 
    \begin{algorithmic}[1]

    \State Let the smoothness candidate set for the source model $f_{S}$ as $\mcM_{S} = \{ \frac{Q_{1}}{\log(n_{S})}, \cdots, \frac{Q_{1} N_{1}}{\log(n_{S})} \}$ and the candidate set for the offset model $f_{\delta}$ as $\mcM_{\delta} = \{ \frac{Q_{2}}{\log(n_{T})}, \cdots, \frac{Q_{2} N_{2}}{\log(n_{T})} \}$ for some fixed positive number $Q_{1},Q_{2}$ and integer $N_{1},N_{2}$.

    \State Obtain the adaptive source model $\hat{f}_{S}$ via the training and validation with the Gaussian kernel and $\mcM_{S}$.

    \State Generate the label $\hat{e}_{T,i} = y_{T,i} - \hat{f}_{S}(x_{T,i})$
    and the offset dataset as $\tilde{\mcD}_{T} = \{(x_{T,1}, \hat{e}_{T,1} ),\cdots,(x_{T,n_{T}}, \hat{e}_{T,n_{T}} ))\}$.

    \State Using the offset datasets $\tilde{\mcD}_{T}$ to obtain the adaptive offset model $\hat{f}_{\delta}$ via the training and validation with the Gaussian kernel and $\mcM_{\delta}$.
    
    \end{algorithmic}    
\end{algorithm}

While SATL can be viewed as a specification of the Algorithm~\ref{algo: Two-step TL KRR}, the desirable property exhibited by the imposed Gaussian kernel surpasses all other misspecified kernel choices by enabling estimators always to adapt to the genuine smoothness of the functions inherently even with unknown the true Sobolev smoothness $m_{0},m$.

\begin{remark}
    Although Lepski's method could be a viable alternative in SATL when $\mu(x)$ is known, we opt for the training/validation approach as a more universally applicable and safer option.
\end{remark}


\subsection{Theoretical Analysis}
In order to provide concrete theoretical bounds, we assume the offset function of $f_{T}$ and $f_{S}$ in the range of the $h$-ball of $H^{m}$, i.e. $f_{S}$ is said to be $h$-transferable to $f_{T}$ if $\|g\|_{H^{m}} \leq h$. Hence, the parameter space is defined as 
\begin{equation*}
    \Theta(h,m_0,m) = \{ (\rho_{T}, \rho_{S}): \|f_{S}\|_{H^{m_{0}}}\leq R,  \|f_{T}-f_{S}\|_{H^{m}} \leq h  \}
\end{equation*}
for some positive constant $R$ and $h$. We note that to achieve rigorous optimality in the context of hypothesis transfer learning under the regression setting, such an upper bound for the distance between parameters from both domains is often required, e.g. $\ell^{1}$ or $\ell^{0}$ distance in high-dimensional setting \citep{li2022transfer,tian2022transfer}, Fisher-Rao distance in low-dimensional setting \citep{zhang2022class}, RKHS distance in functional setting \citep{lin2022transfer}, etc.


\begin{theorem}[Optimality of SATL]\label{thm: optimality of SATL}
Under the Assumption~\ref{assump1}, \ref{assump2} and \ref{assump: error tail}, define the constant $\xi(h,f_{S}):= h^2 / \| f_{S} \|_{H^{m_{0}}}^2$, then we have the lower bound for the transfer learning problem and the upper of SATL as follows.

    \begin{enumerate}
        \item (\textbf{Lower bound}) There exists a constant $C$ s.t. 
        \begin{equation*}
            \inf_{\tilde{f}} \sup_{\Theta(h,m_0,m)}   \E_{(\rho_{T},\rho_{S})} \| \tilde{f} - f_{T} \|_{L_{2}}^2\geq C \left( n_{S}^{-\frac{2m_{0}}{2m_{0}+d}} + n_{T}^{-\frac{2m}{2m+d}}\xi(h,f_{S}) \right),
        \end{equation*}
        where $\inf$ is taken over all possible estimators $\tilde{f}$ based on the combination of target and source data.

        \item (\textbf{Upper bound}) Suppose that $\hat{f}_{T}$ is the output of SATL. For $\delta \in (0,1)$, when $n_{S}$ and $n_{T}$ are sufficiently large but still in transfer learning regime, with probability $1-2\delta$, we have 
        \begin{equation*}
            \| \hat{f}_{T} - f_{T} \|_{L_{2}}^{2}  \leq C \left(\log \frac{4}{\delta}\right)^2 \left\{ \left(\frac{n_{S}}{\log n_{S}}\right)^{-\frac{2m_{0}}{2m_{0}+d}} + \left(\frac{n_{T}}{\log n_{T}}\right)^{-\frac{2m}{2m+d}} \xi(h,f_{S})\right\},
        \end{equation*}
        where $C$ is a universal constant independent of $n_{S}$, $n_{T}$, $h$ and $\delta$.
    \end{enumerate}
\end{theorem}
\begin{remark}
    Theorem~\ref{thm: optimality of SATL} indicates the excess risk of SATL consists of two terms, where the first term is the source error and the second term is the offset error. The first term can be viewed as the error of regressing $f_{T}$ with the source sample size, while the second term is the error of regressing the offset function with the target sample size. The upper bound is tight up to logarithm factors, which is a price paid for adaptivity. Note that the upper bound is exactly tight when the Sobolev smoothness $m_{0}$ and $m$ are known.
\end{remark}

In comparison to the convergence rate of the target-only baseline estimator, $n_{T}^{-{2m_{0}}/{(2m_{0} + d)}}$, our results indicate that the transfer learning efficacy depends jointly on the sample size in source domain $n_{S}$, and the factor $\xi(h,f_{S})$. The factor $\xi(h,f_{S})$ represents the relative task signal strength between the source and target domains. Geometrically, one can interpret $\xi(h,f_{S})$ as the factor controlling the angle (thus similarity) between $f_{T}$ and $f_{S}$ within the RKHS. 


\begin{figure}[ht]
    \centering
    \includegraphics[width = 0.33\textwidth]
    {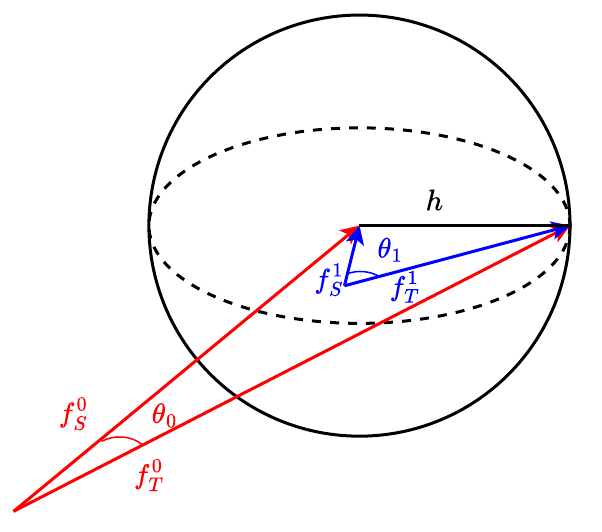}
    \caption{Geometric illustration for how $\xi(h,\mcS)$ will affect the OTL dynamic. The circle represents an RKHS ball centered around $f_{S}$ with radius $h$. Two sets of $f_{S}$ and $f_{T}$ (denoted by red and blue) possess the same offset with the same signal strength $h$ while the source models' signal strength are different, leading to different angle $\theta_{0}$ and $\theta_{1}$ between $f_{S}$ and $f_{T}$.} 
    \label{fig: angle figure}
\end{figure}

Specifically, when $f_{T}$ and $f_{S}$ possess high similarity leading to a small $\xi(h,\mcS)$ (so thus the angle), then the source error will be the dominant term. Thus the convergence rate of $\hat{f}_{T}$ is much faster than the target-only baseline given $n_{S} \gg n_{T}$. On the other hand, $f_{S}$ being adversarial to $f_{T}$ leads to large $\xi(h,f_{S})$, making the offset error the dominant term but still not worse than the target-only baseline up to a constant.

Our results not only recover the analysis presented in \citet{wang2016nonparametric,du2017hypothesis}, but also provide insights into how signal strength from different domains will affect the OTL dynamics. Besides, although statistical rates in \citet{li2022transfer,tian2022transfer} claimed that the OTL is taking effect when the magnitude of signal strength of the offset is small, i.e. small $h$, our results identify it should depend on the angle between $f_{S}$ and $f_{T}$, i.e. $\xi(h,f_{S})$. In Figure~\ref{fig: angle figure}, even signal strength of the offset (i.e. $h$) of $f_{T}^{1}$ and $f_{S}^{1}$ is the same as the offset of $f_{T}^{0}$ and $f_{S}^{0}$, the angle $\theta_{0}$ and $\theta_{1}$ between the target and source functions are different. It indicates the similarity between $f_{S}^{0}$ and $f_{T}^{0}$ is higher and makes the OTL more effective.

Finally, it is also worth highlighting how the refinement of a ``simple'' offset provides a better statistical rate. Based on the saturation effort of KRR, using the same Mat\'ern kernel with smoothness $m_{0}$ for both steps will lead the offset error to be $n_{T}^{- 2\gamma m_{0}/(2\gamma m_{0} + d)  }$, where $\gamma =  \min\{ 2, m/m_{0} \}$, which is never faster than $n_{T}^{-2m/(2m+d)}$.

\section{Experiments}
In this section, we aim to confirm our theoretical results on target-only and transfer learning contexts. 

\subsection{Experiments for Target-Only KRR}\label{subsec: simulation for Adaptivity of Gaussian KRR}
Let $\mcX = [0,1]$ and the marginal distribution of $x$ be the uniform distribution over $[0,1]$. Our objective is to empirically confirm the adaptability of Gaussian kernels in target-only KRR when $f_{0}\in H^{\alpha}([0,1])$ for different smoothness $\alpha$. Specifically, we explore cases where $f_{0}$ belongs to $H^{2}$ and $H^{3}$. To generate such $f_{0}$ with the desired Sobolev smoothness, we set $f_{0}$ to be the sample path that is generated from the Gaussian process with isotropic Mat\'ern covariance kernels $K_{\nu}$ \citep{stein1999interpolation}. We set $\nu = 2.01$ and $3.01$ to generate the corresponding $f_{0}$ with smoothness $2$ and $3$, see Corollary 4.15 in \citet{kanagawa2018gaussian} for detail discussion about the connection between $\nu$ and $\alpha$. Formally, we consider the following data generation procedure: $y_{i} = f_{0}(x_{i}) + \sigma \epsilon_{i}$,
where $\epsilon_{i}$ are i.i.d. standard Gaussian noise, $\{x_{i}\}_{i=1}^{n} \stackrel{i.i.d.}{\sim} U([0,1])$ and $\sigma = 0.5$.

We verify both the nonadaptive and adaptive rate presented in Theorem~\ref{thm: non-adaptive rate of KRR with Gaussian} and~\ref{thm: adaptive rate of single-task GKRR}. The sample size ranges from $1000$ to $3000$ in intervals of $100$. For different $\alpha$, we set $\lambda = \exp\{-Cn^{\frac{2}{2\alpha + 1}}\}$ with a fixed $C$. To evaluate the adaptivity rate, we set the candidate smoothness as $[1,2,3,4,5]$ and split the dataset equally in size to implement training and validation. The generalization error $\|\hat{f} - f\|_{L_{2}}$ is obtained by Simpson's rule. For each combination of $n$ and $\alpha$, we repeat the experiments $100$ times and report the average generalization error. To demonstrate the convergence rate of the error is sharp, we regress the logarithmic average generalization error, i.e. $\log(\|\hat{f} - f\|_{L_{2}})$, on $\log(n)$ and compare the regression coefficient to its theoretical counterpart $-\frac{2\alpha}{2\alpha + 1}$.

\begin{figure}[ht]
    \centering
    \includegraphics[width = \textwidth]{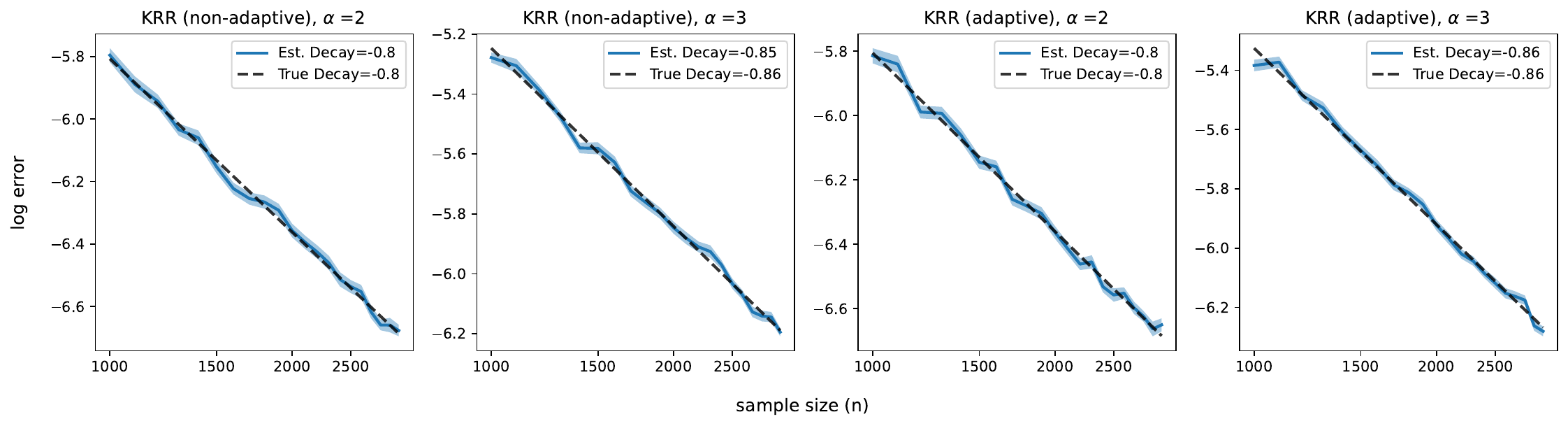}
    \caption{Error decay curves of target-only KRR based on Gaussian kernel, both axes are in log scale. The blue curves denote the average generalization errors over 100 trials. The dashed black lines denote the theoretical decay rates.
    }
    \label{fig: adaptive and nonadaptive rate for target-only KRR with the best slope}
\end{figure}

We try different values of $C$ lies in $[0.05, \cdots,4]$, and report the optimal curve in Figure~\ref{fig: adaptive and nonadaptive rate for target-only KRR with the best slope} under the best choice of $C$. Remarkably, for both nonadaptive and adaptive rates, the theoretical lines align closely with the empirical data points. The estimated regression coefficients also closely agree with the theoretical counterparts. Additionally, we also report the generalization error decay estimation results for other values of $C$ and refer to Appedix~\ref{apd: simulation} for more details.

\subsection{Experiments for Transfer Learning}

We now illustrate our theoretical analysis of SATL through two experiments with synthetic data. We generate the target/source functions and the offset function as follows: $(i)$ The target function $f_{T}$ is a sample path of the Gaussian process with Mat\'ern kernel $K_{1.01}$ such that $f_{T} \in H^{1}$; $(ii)$ The offset function $f_{\delta}$ is a sample path of Gaussian process with Mat\'ern kernel $K_{\nu}$ with $\nu = 2.01, 3.01, 4.01$ such that $f_{\delta}$ belongs to $H^{2}, H^{3}, H^{4}$ respectively. Hence, we consider the following data generation procedure: 
\begin{gather*}
    \{x_{i,T}\}_{i=1}^{n_{T}},  \{x_{i,S}\}_{i=1}^{n_{S}} \stackrel{i.i.d.}{\sim} U([0,1]) \\
    y_{i,T} = f_{T}(x_{i,T}) + \sigma \epsilon_{i,T}, \quad i = 1, \cdots, n_{T} \\
    y_{i,S} = f_{T}(x_{i,S}) + f_{\delta}(x_{i,S}) + \sigma \epsilon_{i,S}, \quad  i = 1, \cdots, n_{S} 
\end{gather*}
where $\epsilon_{i,p}$ are i.i.d. standard Gaussian noise and $\sigma = 0.5$.

To demonstrate the transfer learning effect, we consider two different settings: (1) we fix $n_{T}$ as $50$ and vary $n_{S}$. (2) We set $n_{S} = n_{T}^{3/2}$ while varying $n_{T}$, i.e. the source sample size grows in a polynomial order of target sample size. In the first scenario, it is expected that the generalization error first decreases and then remains unchanged as $n_{S}$ increases since the offset error (a constant for fixed $n_{T}$) eventually dominates. In the second scenario, the generalization error satisfies $ \mcE(\hat{f}_{T}) = O ( n_{T}^{-\frac{3m_{0}}{2m_{0} + 1}} + n_{T}^{-\frac{2m}{2m+1}} \xi(h,f_{S}) ) = O(n_{T}^{-\frac{2m}{2m+1}})$. We also consider the finite basis expansion (FBE) TL algorithm proposed in \citet{wang2016nonparametric} as a competitor. The basis that the authors originally used in their paper was the Fourier basis which produces weak results in our setting, and we compared it to a modification of their algorithm with Bspline as the final competitor. We refer to Appendix~\ref{apd: simulation} for additional results on the Fourier basis and the other experiments' details.

\begin{figure}[ht]
\centering
\begin{subfigure}{.5\textwidth}
  \centering
  \includegraphics[width=\linewidth]{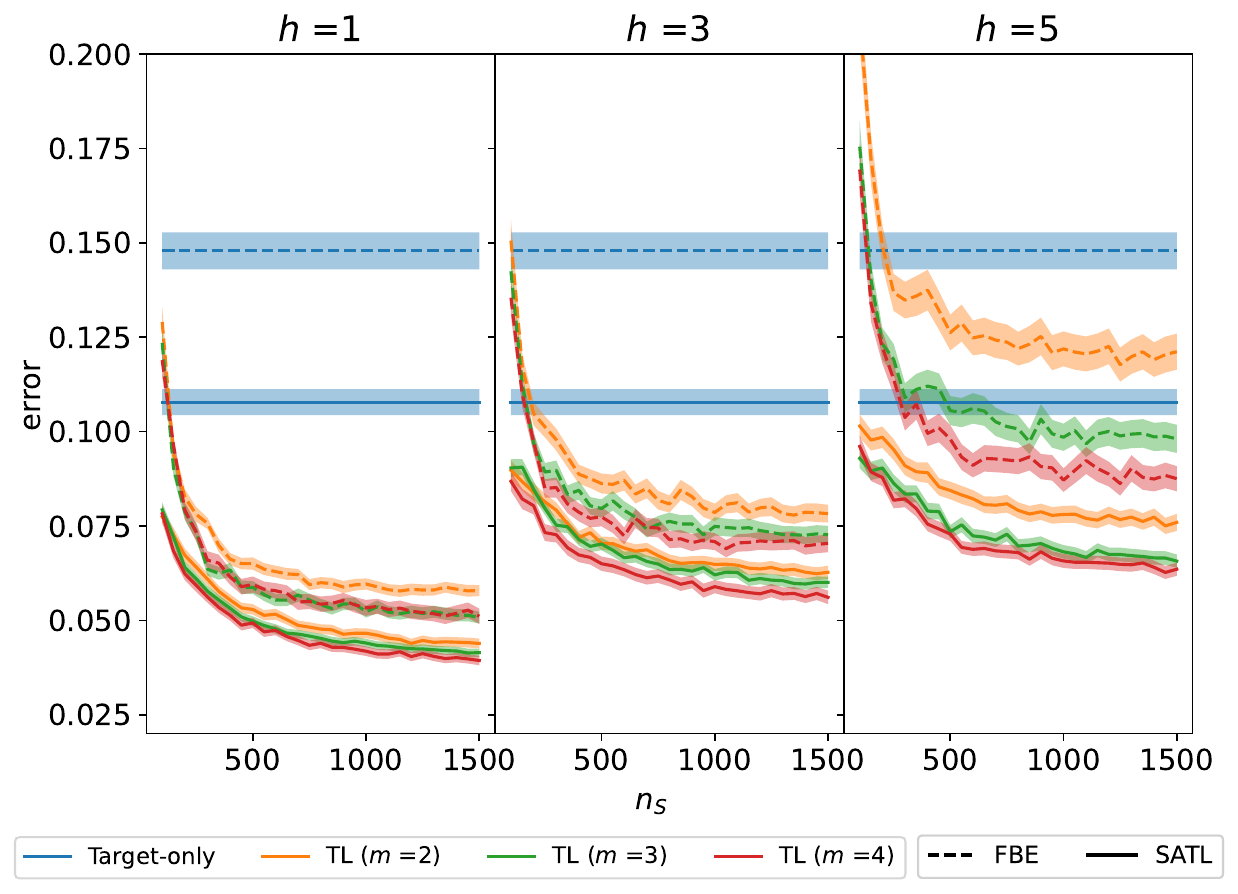}
  \caption{}
  \label{fig:sub1}
\end{subfigure}%
\begin{subfigure}{.5\textwidth}
  \centering
  \includegraphics[width=\linewidth]{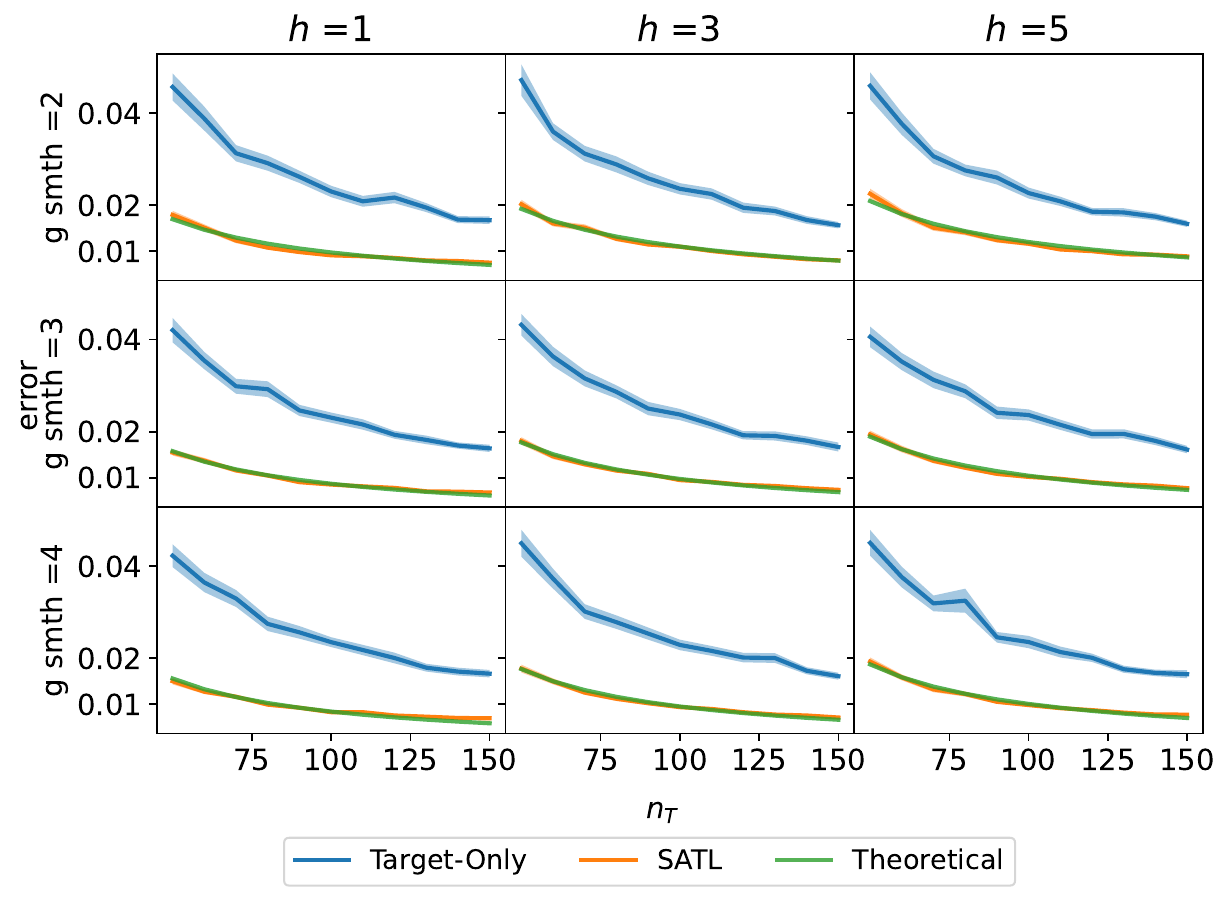}
  \caption{}
  \label{fig:sub2}
\end{subfigure}
\caption{Generalization error under different $h$ and smoothness of $f_{\delta}$. Each curve denotes the average error over 100 trails and the shadow regions denote one standard error of the mean. The left figure contains results for fixed $n_{T}$ scenario while the right figure is for varying $n_{T}$ scenario.}
\label{fig:test}
\end{figure}

Figure~\ref{fig:sub1} presents the generalization error for the fixed $n_{T}$ scenario. As $n_{S}$ increases, the generalization error initially decreases and then gradually levels off, consistent with our expectations. Furthermore, if the smoothness of the offset function $f_{\delta}$ is higher, a smaller error is obtained, which agrees mildly with our theoretical analysis. Finally, compared to the FBE approach, SATL achieves overall smaller errors. Figure~\ref{fig:sub2} presents the generalization error for $n_{S} = n_{T}^{3/2}$ with varying $n_{T}$ setting. Here the error term is expected to be upper bounded by $n_{T}^{-2m/(2m+1)}$. One can see our empirical error is consistent with the theoretical upper bound asymptotically in all settings. Besides, the SATL outperforms the target-only learning KRR baseline in all settings.

\section{Discussion}
We presented SATL, a kernel-based OTL that uses Gaussian kernels as imposed kernels. This enables the estimators to adapt to the varying and unknown smoothness in their corresponding functions. SATL achieves minimax optimality (up to a logarithm factor) as the upper bound of SATL matched the lower bound of the OTL problem.
Notably, our Gaussian kernels' result in target-only learning also serves as a good supplement to misspecified kernel learning literature.

Focusing on future work, when the source model possesses different smoothness compared to the target model, it is of interest to see when the source helps in learning the target, especially for a rougher source model. Besides, in developing the optimality of target-only KRR with Gaussian kernels, we use the Fourier transform technique to control the approximation error, which is feasibly applied to RKHS that are norm-equivalent to fractional Sobolev spaces. Although this makes our results quite broadly applicable when one is primarily interested in the smoothness of the functions, it certainly doesn't cover all possible structures of interest (e.g. periodic functions etc). It is of interest to develop other tools to extend the results to more general RKHS.

\bibliographystyle{plainnat}
\bibliography{ref}

\newpage
\appendix
{\noindent \LARGE \textbf{Appendix}} \par 
\noindent 


\section*{Table of Contents}
\startcontents[sections]
\printcontents[sections]{}{1}{}

\newpage











\section{Notation}\label{apd: notation}
The following notations are used throughout the rest of this work and follow standard conventions. For asymptotic notations: $f(n) = O(g(n))$ means for all $c$ there exists $k>0$ such that $f(n)\leq c g(n)$ for all $n\geq k$; $f(n)\asymp g(n)$ means $f(n) = O(g(n))$ and $g(n) = O(f(n))$; $f(n) = \Omega(g(n))$ means for all $c$ there exists $k>0$ such that $f(n)\geq c g(n)$ for all $n\geq k$. We use the asymptotic notations in probability $O_{\mathbb{P}}(\cdot)$. That is, for a positive sequence $\{a_{n}\}_{n\geq 1}$ and a non-negative random variable sequence $\{X_{n}\}_{n\geq 1}$, we say $X_{n} = O_{\mathbb{P}}(a_n)$ if for any $\delta> 0$, there exist $M_{\delta}$ and $N_{\delta}$ such that $\mathbb{P}(X_{n}\leq M_{\delta}a_{n})\geq 1-\delta$, $\forall n\geq N_{\delta}$. The definition of $\Omega_{\mbP}$ follows similarly.


For a function $f\in L_{1}(\mathbb{R}^{d})$, its Fourier transform is denoted as 
\begin{equation*}
    \mcF(f)(\omega)=(2 \pi)^{-d / 2} \int_{\mathbb{R}^d} f(x) e^{-i x^T \omega} d x.
\end{equation*}

Since we assume $\mu_{T} = \mu_{T}$, we use $L_{2}(\mcX, d\mu_{p})$ for $p\in \{T,S\}$ to represent the Lebesgue $L_{2}$ space and abbreviate it as $L_{2}$ for simplicity when there is no confusion.

\section{Foundation of RKHS}\label{apd: the foundation of RKHS}

\subsection{Basic Concept}\label{apd: RKHS basic concept}
In this section, we will present some facts about the RKHS that are useful in our proof and refer readers to \citet{wendland2004scattered} for a more detailed discussion. 

Assume $K: \mcX \times \mcX \rightarrow \mathbb{R}$ is a continuous positive definite kernel function defined on a compact set $\mcX \subset \mathbb{R}^{d}$ (with positive Lebesgure measure and Lipschitz boundary). Indeed, every positive definite kernel can be associated with a reproducing kernel Hilbert space (RKHS). The RKHS, $\mcH_{K}$, of $K$ are usually defined as the closure of linear space $\text{span}\{ K(\cdot, x), x\in \mcX \}$. In a special case where the kernel function $K(x,y)$ is equal to a translation invariant (stationary) function $\Phi(x-y) = K(x,y)$ with $\Phi:\mbR^{d}\rightarrow \mbR$, we can characterize the RKHS of $K$ in terms of Fourier transforms, i.e. 
\begin{equation*}
    \mcH_{K}(\mbR^{d}) = \left\{ f \in L_{2}(\mathbb{R}^{d}) \cap C(\mathbb{R}^{d}) : \frac{\mcF(f)}{\sqrt{\mcF(K)}} \in L_{2}(\mathbb{R}^{d}) \right\}.
\end{equation*}
When $\mcX$ is a subset of $\mbR^{d}$, such a definition still captures the regularity of functions in $\mcH_{K}(\mcX)$ via a norm equivalency result that holds as long as X has a Lipschitz boundary.

For an integer $m$, we introduce the integer-order Sobolev space, $\mcW^{m, p}(\mcX)$. For vector $\alpha = (\alpha_{1},\cdots,\alpha_{d})$, define $|\alpha| = \alpha_{1} + \cdots + \alpha_{d}$ and $D^{(\alpha)} = \frac{\partial^{|\alpha|}}{\partial x_{1}^{\alpha_{1}} \cdots \partial x_{d}^{\alpha_{d}}}$ denote the multivariate mixed partial weak derivative. Then 
\begin{equation*}
    \mcW^{m, p}(\mcX)  = \left\{ f\in L^{p}(\mcX): D^{(\alpha)}f \in L_{p}(\mbR^{d}), \forall |\alpha| \leq m \right\},
\end{equation*}
where $m$ is the smoothness order of the Sobolev space. In this paper, we only consider $p=2$ and abbreviate $\mcW^{m,2}(\mcX):= H^{m}(\mcX)$. Later, in Appendix~\ref{apd: Norm Equivalency between RKHS and Sobolev Space}, we will see one can define the $H^{m}$ via Fourier transform of the reproducing kernel instead of weak derivative.

We now introduce the power space of an RKHS. For the reproducing kernel $K$, we can define its integral operator $T_{K}: L_{2} \rightarrow L_{2}$ as 
\begin{equation*}
    T_{K} (f)(\cdot) = \int_{\mcT} K(s,\cdot)f(s) ds.
\end{equation*}
$L_{K}$ is self-adjoint, positive-definite, and trace class (thus Hilbert-Schmidt and compact). By the spectral theorem for self-adjoint compact operators, there exists an at most countable index set $N$, a non-increasing summable positive sequence $\{\tau_{j}\}_{j\geq 1}$ and an orthonormal basis of $L_{2}$, $\{e_{j}\}_{j\geq1}$ such that the integrable operator can be expressed as 
\begin{equation*}
    T_{K}(\cdot) = \sum_{j\in N} \tau_{j} \langle \cdot, e_{j} \rangle_{L_{2}} e_{j}.
\end{equation*}
The sequence $\{\tau_{j}\}_{j\geq 1}$ and the basis $\{e_{j}\}_{j\geq1}$ are referred as the eigenvalues and eigenfunctions. The Mercer's theorem shows that the kernel $K$ itself can be expressed as 
\begin{equation*}
    K(x,x') = \sum_{j\in N} \tau_{j} e_{j}(x) e_{j}(x'), \quad \forall x,x' \in \mcT,
\end{equation*}
where the convergence is absolute and uniform.

We now introduce the fractional power integral operator and the composite integral operator of two kernels. For any $s\geq 0$, the fractional power integral operator $L_{K}^{s}:L_{2} \rightarrow L_{2}$ is defined as 
\begin{equation*}
    T_{K}^{s}(\cdot) = \sum_{j\in N} \tau_{j}^{s} \langle \cdot, e_{j} \rangle_{L_{2}} e_{j}.
\end{equation*}
Then the power space $[\mcH_{K}]^{s}$ is defined as 
\begin{equation*}
    [\mcH_{K}]^{s} := \left\{ \sum_{j\in N} a_{j} \tau_{j}^{\frac{s}{2}} e_{j}: (a_{j})\in \ell^{2}(N)  \right\}
\end{equation*}
and equipped with the inner product 
\begin{equation*}
    \langle f, g \rangle_{[\mcH_{K}]^{s}} = \left \langle T_{K}^{-\frac{s}{2}}(f), T_{K}^{-\frac{s}{2}}(g) \right\rangle_{L_{2}}.
\end{equation*}
For $0 < s_{1} < s_{2}$, the embedding $[\mcH_{K}]^{s_{2}} \hookrightarrow [\mcH_{K}]^{s_{1}}$ exists and is compact. A higher $s$ indicates the functions in $[\mcH_{K}]^{s}$ have higher regularity. When $\mcH_{K} = H^{m}$ with $m > d/2$, the real interpolation indicates $ [H^{m}]^{s} \cong H^{m s}, \forall s > 0$.

\subsection{Norm Equivalency between RKHS and Sobolev Space}\label{apd: Norm Equivalency between RKHS and Sobolev Space}
Now, we state the result that connects the general RKHS and Sobolev space.
\begin{lemma}\label{lemma: equivalence between RKHS and Sobolev}
    Let $K(x,x')$ be the translation-invariant kernel and $\tilde{K}\in L^{1}(\mbR^{d})$. Suppose $\mcX$ has a Lipschitiz boundary, and the Fourier transform of $K$ has the following spectral density of $m$, for $m\geq d/2$, 
    \begin{equation}\label{eqn: Fourier Transform of reproducing kernel}
        c_{1}(1 + \|\cdot\|_{2}^{2})^{m} \leq \mcF(K)(\cdot) \leq c_{2}(1 + \|\cdot\|_{2}^{2})^{m}.
    \end{equation}
    for some constant $0 < c_{1} \leq c_{2}$. Then, the associated RKHS of $K$, $\mcH_{K}(\mcX)$, is norm-equivalent to the Sobolev space $H^{m}(\mcX)$.
\end{lemma}
Hence, we can naturally define the Sobolev space of order $m$ ($m>\frac{d}{2}$) as 
\begin{equation*}
    H^{m}(\mathbb{R}^{d}) = \left \{ f \in L_{2}(\mathbb{R}^{d}) \cap C(\mathbb{R}^{d}) : \mcF(f)(\cdot) (1 + \|\cdot\|_{2}^{2})^{m} \in L_{2}(\mathbb{R}^{d})    \right\}.
\end{equation*}
One advantage of this definition over the classical way that involves weak
derivatives is it does not require $m$ to be an integer, and thus one can consider the fractional Sobolev space, i.e. $m\in \mbR^{+}$. Such equivalence also holds on $\mcX$ by applying the extension theorem \citep{devore1993besov}. As an implication, let $K_{m,\nu}$ denotes the isotropic Mat\'ern kernel \citep{stein1999interpolation}, i.e. 
\begin{equation*}
    K_{m,\nu}(x ;\rho) = \frac{2^{1-\nu}}{\Gamma(\nu) }\left( \sqrt{2\nu} \frac{\|x\|_2}{\rho}\right)^\nu K_\nu\left( \sqrt{2\nu} \frac{\|x\|_2}{\rho}\right),
\end{equation*}
then the Fourier transform of $K_{m,\nu}$ satisfies Equation~\ref{eqn: Fourier Transform of reproducing kernel} with $m = \nu + \frac{d}{2}$, and thus the RKHS associated with $K_{m,\nu}$ is norm equivalent to Sobolev space $H^{\nu + \frac{d}{2}}$ \citep{wendland2004scattered}. 

For a reproducing kernel that satisfies (\ref{eqn: Fourier Transform of reproducing kernel}), we call it a kernel with Fourier decay rate $m$ and denote it as $K_{m}$. We further denote its associated RKHS as $\mcH_{K_{m}}(\mcX)$. The Fourier decay rate $m$ captures the regularity of $\mcH_{K_{m}}(\mcX)$.

Now, we are ready to define the function space of $f_{S}$, $f_{T}$ and $f_{\delta}$ via the kernel regularity
\begin{assumption}[Smoothness of Target/Source]\label{assump: general assump1}
     There exists an $m_0 \geq d/2$ such that $f_{T}$ and $f_{S}$ belong to $\mcH_{K_{m_{0}}}$.
\end{assumption}
\begin{assumption}[Smoothness of Offset]\label{assump: general assump2}
There exists an $m\geq m_{0}$ such that $f_{\delta} := f_{T} - f_{S}$ belongs to $\mcH_{K_{m}}$.  
\end{assumption}

The proof of all the theoretical results in Section \ref{sec: target-KRR} and \ref{sec: SATL} is built on the assumptions that the true functions are in Sobolev space. Via the norm equivalency (Lemma~\ref{lemma: equivalence between RKHS and Sobolev}), the true functions also reside in RKHSs associated with kernel $K_{m_{0}}$ and $K_{m}$. Therefore, all the theoretical results still hold under Assumption~\ref{assump: general assump1} and \ref{assump: general assump2}.

\section{Target-Only KRR Learning Results}
\subsection{Comparision to Previous Work}\label{apd: target-only literature compare}
In Table~\ref{table: convergence rate comparison}, we compare our results with some state-of-the-art works (to the best of our knowledge) that consider general/Mat\'ern misspecified kernels and Gaussian kernels in target-only setting KRR. For a detailed review of the optimality of misspecified KRR, we refer readers to \citet{zhang2023optimality}. 

\citet{wang2022gaussian} considered the true function lies in $H^{\alpha_{0}}$ while the imposed kernels are misspecified Mat\'ern kernels. On the other hand, \citet{zhang2023optimality} considered the minimax optimality for misspecified KRR in general RKHS, i.e. the imposed kernel is $K$ while the true function $f_{0}\in [\mcH_{K}]^{s}$ for $s\in (0,2]$. However, when the RKHS is specified as the Sobolev space, the results in both papers are equivalent by applying the real interpolation technique in Appendix~\ref{apd: the foundation of RKHS}. Therefore, we place them in the same row.

Unlike the necessary conditions that the imposed RKHS must fulfill $\alpha_{0}' > \alpha_{0}$ to achieve optimality \citep{wang2022gaussian,zhang2023optimality}, our results circumvent this requirement, thereby being more robust. Compared to other works on Gaussian kernel-based KRR, our result shows that the optimality can be achieved only via a fixed bandwidth Gaussian kernel.

We also want to highlight the technical challenge that lies in handling the approximation error. When the true function $f_{0}$ belongs to $H^{\alpha_{0}}$, one can expand the intermediate term $f_{\lambda}$ (an element in the imposed RKHS $H^{\alpha_{0}'}$) and $f_{0}$ under the same basis and controls the approximation error in the form of $\lambda^{\alpha_{0}/\alpha_{0}'}\|f_{0}\|_{H^{\alpha_{0}}}$. This makes the optimal decay order of $\lambda$ in $n$ take a polynomial pattern like the misspecified kernel case in Proposition~\ref{proposition: target-only learning}. However, since the imposed kernel $K$ is the Gaussian kernel and the intermediate term $f_{\lambda}$ is an element in $\mcH_{K}$ (see the definition below), such techniques are no longer applicable. To address this, we leverage the Fourier transform of the true and imposed kernel to control the approximation error.

\begin{table*}[ht]
\centering
\renewcommand*{\arraystretch}{1.5}
\caption{Comparison of generalization error convergence rate (non-adaptive) between our result and the prior literature. Here, we assume the mean function $f_0$ belongs to Sobolev space $H^{\alpha_{0}}$, imposed RKHS means the RKHS that $\hat{f}$ belongs to. ``$-$'' in column $\gamma$ means the bandwidth is fixed during training and does not have an optimal order in $n$. $\mcH_{K}$ means the RKHS associated with the Gaussian kernel while $H^{\alpha_{0}'}$ means the Sobolev space with smoothness order $\alpha_{0}'$.}
\begin{tabular}{|c|c|c|c|c|} 
\hline
  \multicolumn{1}{|c|}{Paper} & \multicolumn{1}{|c|}{Imposed RKHS}  &  \multicolumn{1}{|c|}{Rate} & \multicolumn{1}{|c|}{$\lambda$} & \multicolumn{1}{|c|}{$\gamma$} \\
\hline 
\citet{wang2022gaussian}, \citet{zhang2023optimality}& $H^{\alpha_{0}'}, \alpha_{0}'>\frac{\alpha_{0}}{2}$ & $n^{-\frac{2\alpha_{0}}{2\alpha_{0} + d}}$ & $n^{-\frac{2\alpha_{0}'}{2\alpha_{0} + d}}$ & $-$  \\
\hline 
\citet{eberts2013optimal}& $\mcH_{K}$ & $n^{-\frac{2\alpha_{0}}{2\alpha_{0} + d}+ \xi}, \forall \xi> 0$ & $n^{-1}$ & $n^{-\frac{1}{2\alpha_{0} + d} }$ \\
\hline
\citet{hamm2021adaptive}& $\mcH_{K}$ & $n^{-\frac{2\alpha_{0}}{2\alpha_{0} + d}}\log^{d+1}(n)$ & $n^{-1}$ & $n^{-\frac{1}{2\alpha_{0} + d} }$ \\
\hline
This work & $\mcH_{K}$ & $n^{-\frac{2\alpha_{0}}{2\alpha_{0} + d}}$ & $ \exp\{-Cn^{\frac{2}{2\alpha_{0} + d}}\} $ & $-$  \\
\hline
\end{tabular}

\label{table: convergence rate comparison}

\end{table*}

\subsection{Proof of Proof of Non-adaptive Rate (Theorem~\ref{thm: non-adaptive rate of KRR with Gaussian})}\label{apd: proof of non-adaptive rate}
In the following proof, we will use $C$, $C_1$, and $C_2$ to represent universal constants that could change from place to place. We also omit the $\mcX$ in the norms or in the inner product notation unless specifically specified. We use $\|\cdot\|_{op}$ to denote the operator norm of a bounded linear operator. For the given imposed Gaussian kernel $K$, we define its corresponding integral operator $T_{K}:L_{2}(\mcX) \rightarrow L_{2}(\mcX)$ as 
\begin{equation*}
    T_{K}(f)(\cdot) = \int_{\mcX} f(x) K(x,\cdot) d\rho(x)
\end{equation*}
where $\rho(x)$ is the probability measure over $\mcX$. Note that the integral operator $T_{K}$ can also be seen as a bounded linear operator on $\mcH_{K}$. Its empirical version is 
\begin{equation*}
    T_{K,n} (f) (\cdot) = \frac{1}{n} \sum_{i=1}^{n} \langle f, K_{x_{i}} \rangle_{\mcH_{K}} K_{x_i}
\end{equation*}
where $\mcH_{K}$ is the RKHS of imposed kernel $K$. Besides, for a $x\in \mcX$, we define the sampling operator as $K_{x}: \mcH_{K} \rightarrow \mbR, f \mapsto \langle f, K_{x} \rangle_{\mcH_{K}}$ and its adjoint operator $K_{x}^{*}: \mbR \rightarrow \mcH_{K}, y \mapsto yK_{x}$. Then 
\begin{equation*}
    T_{K,x} = K_{x}^{*} K_{x}, \quad \text{and} \quad T_{K,n} = \frac{1}{n}\sum_{i=1}^{n}K_{x_{i}}^{*} K_{x_{i}}.
\end{equation*}
In addition, we denote the effective dimension
\begin{equation*}
    \mcN(\lambda) = tr( (T_{K} + \lambda)^{-1}T_{K} ) = \sum_{j=1}^{\infty} \frac{s_{j}}{s_{j}+\lambda}.
\end{equation*}
With the notation, the KRR estimator can be written as 
\[
\hat{f} = (T_{K,n} + \lambda \mathbf{I})^{-1} ( \frac{1}{n}\sum_{i=1}^{n}K_{x_{i}}y_{i} ) := (T_{K,n} + \lambda \mathbf{I})^{-1} g_{n}
\]
where $g_n = \frac{1}{n}\sum_{i=1}^{n} K_{x_i}y_{i}$ and $\mathbf{I}$ is the identical operator.
We further define the intermediate term as follows,
\begin{equation*}
     f_{\lambda} : = 
    \underset{  f \in \mcH_{K}}{\operatorname{argmin}} \left \{  \|  ( f_{0} - f )  \|_{L_{2}}^{2} + \lambda \|f\|_{\mcH_{K}}^2  \right\} 
\end{equation*}
and it is not hard to show $f_{\lambda} = (T_{K}+ \lambda \mathbf{I})^{-1} T_{K}(f_{0}) = (T_{K}+ \lambda \mathbf{I})^{-1} g$.

\subsubsection{Proof of the approximation error}
For the approximation error, we can directly apply Proposition~\ref{prop: bounds for approximation error}, which leads to 
\begin{equation*}
    \|f_{\lambda} - f_{0}\|_{L_{2}}^{2} \leq log(\frac{1}{\lambda})^{-\alpha_{0}} \|f_{0}\|_{H^{\alpha_{0}}}^{2}. 
\end{equation*}
Then selecting $\log(1/\lambda) \asymp n^{\frac{2}{2\alpha_{0}+d}}$ leads to 
\begin{equation*}
        \|f_{\lambda} - f_{0}\|_{L_{2}}^{2} \leq n^{-\frac{2\alpha_{0}}{2\alpha_{0}+d}} \|f_{0}\|_{H^{\alpha_{0}}}^{2}. 
\end{equation*}

\subsubsection{Proof of the estimator error}
\begin{theorem}\label{thm: bounds for estimation error}
    Suppose the Assumption (A1) to (A3) hold and $\|f_{0}\|_{L_{q}}\leq C_{q}$ for some $q$. Then by choosing $\log(1/\lambda)\asymp n^{\frac{2}{2\alpha_{0}+d}}$, for any fixed $\delta \in (0,1)$, when $n$ is sufficient large, with probability $1-\delta$, we have 
    \begin{equation*}
        \left \| \hat{f} - f_{\lambda}  \right\|_{L_{2}} \leq ln(\frac{4}{\delta}) C n^{-\frac{\alpha_{0}}{2\alpha_{0}+d}}
    \end{equation*}
    where $C$ is a constant proportional to $\sigma$. 
\end{theorem}
\begin{proof}
First, we notice that 
\begin{equation*}
\begin{aligned}
        \left\| \hat{f} - f_{\lambda} \right\|_{L_{2}} & = \left\| T_{K}^{\frac{1}{2}} \left(  \hat{f} - f_{\lambda} \right) \right\|_{\mcH_{K}} \\
        & \leq \underbrace{\left\|T_{K}^{\frac{1}{2}} \left( T_{K} + \lambda I  \right)^{-\frac{1}{2}} \right\|_{op}}_{A_{1}} \\
        & \quad \cdot  \underbrace{\left\| \left( T_{K} + \lambda I  \right)^{\frac{1}{2}} \left( T_{K,n} + \lambda I  \right)^{-1} \left( T_{K} + \lambda I  \right)^{\frac{1}{2}} \right\|_{op}}_{A_{2}} \\
        & \quad \cdot \underbrace{\left\| \left( T_{K} + \lambda I  \right)^{-\frac{1}{2}} \left( g_{n}  - \left( T_{K,n} + \lambda I  \right)^{-1}f_{\lambda} \right)  \right\|_{\mcH_{K}}}_{A_{3}}
\end{aligned}
\end{equation*}
For the first term $A_{1}$, we have 
\begin{equation*}
    A_{1} = \left\|T_{K}^{\frac{1}{2}} \left( T_{K} + \lambda I  \right)^{-\frac{1}{2}} \right\|_{op} = \sup_{i\geq 1} \left( \frac{s_{j}}{s_{j} + \lambda}\right)^{\frac{1}{2}} \leq 1.
\end{equation*}
For the second term, using Proposition~\ref{prop: bounds for A2} with sufficient large $n$, we have 
\begin{equation*}
    v := \frac{\mcN(\lambda)}{n} ln ( \frac{8\mcN(\lambda)}{\delta} \frac{(\|T_{K}\|_{op}+\lambda)}{\|T_{K}\|_{op}}) \leq \frac{1}{8}
\end{equation*}
such that
\begin{equation*}
     \left\| \left( T_{K} + \lambda I \right)^{-\frac{1}{2}} \left( T_{K} - T_{K,n} \right) \left( T_{K} + \lambda I \right)^{-\frac{1}{2}} \right\|_{op} \leq \frac{4}{3}v + \sqrt{2v} \leq \frac{2}{3}
\end{equation*}
holds with probability $1-\frac{\delta}{2}$. Thus,
\begin{equation*}
\begin{aligned}
A_{2} & = \left\| (T_{K} + \lambda I)^{\frac{1}{2}} (T_{K,n} + \lambda I^{-1}) (T_{K} + \lambda I )^{\frac{1}{2}}\right\|_{op}\\
& =\left\|\left( (T_{K} + \lambda I)^{-\frac{1}{2}}\left(T_{K,n}+\lambda\right) (T_{K} + \lambda I) ^{-\frac{1}{2}}\right)^{-1}\right\|_{op} \\
& =\left\|\left(I-  (T_{K} + \lambda I )^{-\frac{1}{2}}\left(T_{K,n}-T_{K}\right) (T_{K} + \lambda I )^{-\frac{1}{2}}\right)^{-1}\right\|_{op} \\
& \leq \sum_{k=0}^{\infty}\left\| (T_{K} + \lambda I )^{-\frac{1}{2}}\left(T_{K} - T_{K,n}\right) (T_{K} + \lambda I )^{-\frac{1}{2}}\right\|_{op}^k \\
& \leq \sum_{k=0}^{\infty}\left(\frac{2}{3}\right)^k \leq 3,
\end{aligned}
\end{equation*}
For the third term $A_{3}$, notice 
\begin{equation*}
    \left\| \left( T_{K} + \lambda I  \right)^{-\frac{1}{2}} \left( g_{n}  - \left( T_{K,n} + \lambda I  \right)^{-1}f_{\lambda} \right)  \right\|_{\mcH_{K}} = \left\| \left( T_{K} + \lambda I  \right)^{-\frac{1}{2}} \left[ (g_{n} - T_{K,n}(f_{\lambda})) - (g - T_{K}(f_{\lambda})) \right] \right\|_{\mcH_{K}}
\end{equation*}
Using the Proposition~\ref{prop: bounds for A3}, with probability $1-\frac{\delta}{2}$, we have 
\begin{equation*}
     T_{3} = \left\| \left( T_{K} + \lambda I  \right)^{-\frac{1}{2}} \left( g_{n}  - \left( T_{K,n} + \lambda I  \right)^{-1}f_{\lambda} \right)  \right\|_{\mcH_{K}} \leq C ln(\frac{4}{\delta}) n^{-\frac{\alpha_{0}}{2\alpha_{0}+d}} 
\end{equation*}
where $C \propto \sigma$. Combing the bounds for $T_{1}$, $T_{2}$ and $T_{3}$, we finish the proof. 
\end{proof}

\subsubsection{Proof of Theorem~\ref{thm: non-adaptive rate of KRR with Gaussian}}
Based on the triangle inequality and the approximation/estimation error, we finish the proof as follows,
\begin{equation*}
\begin{aligned}
    \left\| \hat{f} - f_{0} \right\|_{L_{2}} 
    & \leq \left\| \hat{f} - f_{\lambda} \right\|_{L_{2}} + \left\| f_{\lambda} - f_{0} \right\|_{L_{2}} \\
    & = O_{\mathbb{P}} \left\{  \left(\sigma + \|f_{0}\|_{H^{\alpha_{0}}} \right) n^{-\frac{\alpha_{0}}{2\alpha_{0}+d}}  \right\}.
\end{aligned}
\end{equation*}

\subsubsection{Propositions}

\begin{proposition}\label{prop: bounds for approximation error}
    Suppose $f_{\lambda}$ is defined as follows,
    \begin{equation*}
        f_{\lambda} = \underset{  f \in \mcH_{K}(\mcX)}{\operatorname{argmin}} \left\{ \| f - f_{0} \|_{L_{2}(\mcX)}^{2} + \lambda \| f \|_{\mcH_{K}(\mcX)}^{2} \right\}.
    \end{equation*}
    Then, under the regularized conditions, the following inequality holds,
    \[
    \| f_{\lambda} - f_{0} \|_{L_{2}(\mcX)}^{2} + \lambda \|f_{\lambda}\|_{\mcH_{K}(\mcX)}^{2} \leq C  log\left(\frac{1}{\lambda}\right)^{-\alpha_{0}} \| f_{0} \|_{H^{\alpha_{0}}}^{2}.
    \]
\end{proposition}

\begin{proof}
    Since $\mcX$ has Lipschitz boundary, there exists an extension mapping from $L_{2}(\mcX)$ to $L_{2}(\mathbb{R}^{d})$, such that the smoothness of functions in $L_{2}(\mcX)$ get preserved. Therefore, there exist constants $C_{1}$ and $C_{2}$  such that for any function $g \in H^{\alpha_{0}}(\mcX)$, there exists an extension of $g$, $g_{e}\in H^{\alpha_{0}}(\mathbb{R}^{d})$ satisfying 
    \[
    C_{1} \|g_{e}\|_{H^{\alpha_{0}}(\mathbb{R}^{d})} \leq \|g\|_{H^{\alpha_{0}}(\mcX)} \leq C_{2} \|g_{e}\|_{H^{\alpha_{0}}(\mathbb{R}^{d})}.
    \]
    Denote 
    \[
    f_{\lambda, e} = \underset{  f \in \mcH_{K}(\mathbb{R}^{d})}{\operatorname{argmin}} \left\{ \| f - f_{0,e} \|_{L_{2}(\mathbb{R}^{d})}^{2} + \lambda \| f \|_{\mcH_{K}(\mathbb{R}^{d})}^{2} \right\}
    \]
    Then we have,
    \begin{equation*}
    \begin{aligned}
        \| f_{\lambda} - f_{0} \|_{L_{2}(\mcX)}^{2} + \lambda \|f_{\lambda}\|_{\mcH_{K}(\mcX)}^{2} & \leq 
        \| f_{\lambda, e}|_{\mcX} - f_{0} \|_{L_{2}(\mcX)}^{2} + \lambda \|f_{\lambda, e}|_{\mcX}\|_{L_{2}(\mcX)}^{2} \\
        & \leq C_{2} \left( \| f_{\lambda,e} - f_{0,e} \|_{L_{2}(\mathbb{R}^{d})}^{2} + \lambda \|f_{\lambda, e}\|_{L_{2}(\mathbb{R}^{d})}^{2} \right).
    \end{aligned}
    \end{equation*}
    where $f_{\lambda,e}|_{\mcX}$ is the restriction of $f_{\lambda,e}$ on $\mcX$. By Fourier transform of the Gaussian kernel and Plancherel Theorem, we have 
    \begin{equation*}
    \begin{aligned}
        & \left\|f_{\lambda, e}-f_{0,e}\right\|_{L_2(\mathbb{R}^{d})}^2+\lambda\left\|f_{\lambda,e}\right\|_{\mcH_{K}(\mathbb{R}^{d})}^2 \\
        = & \int_{\mathbb{R}^{d}}\left|\mathcal{F}\left(f_{0,e}\right)(\omega)-\mathcal{F}\left(f_{\lambda, e}\right)(\omega)\right|^2 d \omega+\lambda \int_{\mathbb{R}^{d}}\left|\mathcal{F}\left(f_{\lambda, e}\right)(\omega)\right|^2exp\{C\|\omega\|_{2}^{2}\} d \omega \\
        = & \int_{\mathbb{R}^{d}}\left(\left|\mathcal{F}\left(f_{0,e}\right)(\omega)-\mathcal{F}\left(f_{\lambda,e}\right)(\omega)\right|^2+\lambda\left|\mathcal{F}\left(f_{\lambda,e}\right)(\omega)\right|^2exp\{C\|\omega\|_{2}^{2}\}\right) d \omega \\
        = & \int_{\mathbb{R}^{d}} \frac{\lambda exp\{C\|\omega\|_{2}^{2}\}}{1+\lambda exp\{C\|\omega\|_{2}^{2}\}}\left|\mathcal{F}\left(f_{0,e}\right)(\omega)\right|^2 d \omega \\
        \leq & \int_{\Omega} \frac{\lambda exp\{C(1+\|\omega\|_{2}^{2})\}}{1+\lambda exp\{C(1+\|\omega\|_{2}^{2})\}}\left|\mathcal{F}\left(f_{0,e}\right)(\omega)\right|^2 d \omega+\int_{\Omega^C} \frac{\lambda exp\{C(1+\|\omega\|_{2}^{2})\}}{1+\lambda exp\{C(1+\|\omega\|_{2}^{2})\}}\left|\mathcal{F}\left(f_{0,e}\right)(\omega)\right|^2 d \omega \\
        \leq & \int_{\Omega} \lambda exp\{C(1+\|\omega\|_{2}^{2})\}\left|\mathcal{F}\left(f_{0,e}\right)(\omega)\right|^2 d \omega+\int_{\Omega^C}\left|\mathcal{F}\left(f_{0,e}\right)(\omega)\right|^2 d \omega
    \end{aligned}
    \end{equation*}
    where $\Omega= \{ \omega : \lambda exp\{C(1+\|\omega\|_{2}^{2})\} <1  \}$ and $\Omega^{C} =  \mathbb{R}^{d} \backslash \Omega$, and the third equality follows the definition of $f_{e}^{*}$.
    Over $\Omega^{C}$, we notice that 
    \begin{equation*}
        (1+\|\omega\|_{2}^{2}) \geq \frac{1}{C}log\left( \frac{1}{\lambda} \right) \implies C^{\alpha_{0}} log\left( \frac{1}{\lambda} \right)^{-\alpha_{0}} (1+\|\omega\|_{2}^{2})^{\alpha_{0}} \geq 1.
    \end{equation*}
    Over $\Omega$, we first note that the function $h(\omega) = exp\{ C(1+\|\omega\|_{2}^{2})\} / (1+\|\omega\|_{2}^{2})^{\alpha_{0}}$ reaches its maximum $C^{\alpha_{0}} \lambda^{-1} log(\frac{1}{\lambda})^{-\alpha_{0}}$ if $\lambda$ satisfies $\lambda < exp\{-\alpha_{0}\}$ and $\lambda log(\frac{1}{\lambda})^{\alpha_{0}} \leq C^{\alpha_{0}} exp\{-C\}$. One can verify when $\lambda \rightarrow 0$ as $n\rightarrow 0$, the two previous inequality holds. Then
    \begin{equation*}
        \lambda exp\{C(1+\|\omega\|_{2}^{2})\} \leq C^{\alpha_{0}} log\left( \frac{1}{\lambda} \right)^{-\alpha_{0}} (1+\|\omega\|_{2}^{2})^{\alpha_{0}} \quad \forall \omega \in \Omega.
    \end{equation*}
    Combining the inequality over $\Omega$ and $\Omega^{C}$,
    \begin{equation*}
    \begin{aligned}
        & \left\|f_{e}^{*}-f_{0,e}\right\|_{L_2(\mathbb{R}^{d})}^2+\lambda\left\|f_{e}^{*}\right\|_{\mcH_{K}(\mathbb{R}^{d})}^2 \\
        \leq & \int_{\Omega} \lambda exp\{C(1+\|\omega\|_{2}^{2})\}\left|\mathcal{F}\left(f_{0,e}\right)(\omega)\right|^2 d \omega+\int_{\Omega^C}\left|\mathcal{F}\left(f_{0,e}\right)(\omega)\right|^2 d \omega \\
        \leq & C^{\alpha_{0}} log\left( \frac{1}{\lambda} \right)^{-\alpha_{0}} \int_{\mathbb{R}^{d}} (1+\|\omega\|_{2}^{2})^{\alpha_{0}} |\mcF(f_{0,e})(\omega)|^2 d \omega \\
         = & C^{\alpha_{0}} log\left( \frac{1}{\lambda} \right)^{-\alpha_{0}} \|f_{0,e}\|_{H^{\alpha_{0}}(\mathbb{R}^{d})}^{2} \\
        \leq & C^{'} log\left( \frac{1}{\lambda} \right)^{-\alpha_{0}} \|f_{0}\|_{H^{\alpha_{0}}(\mcX)}^{2}
    \end{aligned} 
    \end{equation*}
    which completes the proof.
\end{proof}

\begin{proposition}\label{prop: bounds for A2}
    For all $\delta \in (0,1)$, with probability at least $1-\delta$, we have 
    \begin{equation*}
        \left\| \left( T_{K} + \lambda I \right)^{-\frac{1}{2}} \left( T_{K} - T_{K,n} \right) \left( T_{K} + \lambda I \right)^{-\frac{1}{2}} \right\|_{op} \leq \frac{4\mcN(\lambda)B}{3n} + \sqrt{\frac{2\mcN(\lambda)}{n}B}
    \end{equation*}
    where 
    \begin{equation*}
        B = ln ( \frac{4\mcN(\lambda)}{\delta} \frac{(\|T_{K}\|_{op}+\lambda)}{\|T_{K}\|_{op}}).
    \end{equation*}
\end{proposition}
\begin{proof}
    Denote $A_{i} = (T_{K} + \lambda I )^{-\frac{1}{2}} (T_{K} - T_{K,x_{i}}) (T_{K} + \lambda I )^{-\frac{1}{2}}$, applying Lemma~\ref{lemma: bound for kernel function}, we get 
    \begin{equation*}
    \begin{aligned}
        \|A_{i}\|_{op} & \leq \left\| \left( T_{K} + \lambda I \right)^{-\frac{1}{2}} T_{K,x} \left( T_{K} + \lambda I \right)^{-\frac{1}{2}}\right\|_{op} +  \left\| \left( T_{K} + \lambda I \right)^{-\frac{1}{2}} T_{K,x_{i}} \left( T_{K} + \lambda I \right)^{-\frac{1}{2}}\right\|_{op} \\
        & \leq  2 E_{K}^{2} \mcN(\lambda)
    \end{aligned}
    \end{equation*}
    Notice  
    \begin{equation*}
    \begin{aligned}
        \E A_{i}^2  & \preceq \E \left[ \left( T_{K} + \lambda I \right)^{-\frac{1}{2}} T_{K,x_{i}} \left( T_{K} + \lambda I \right)^{-\frac{1}{2}} \right]^2 \\
        & \preceq E_{K}^{2} \mcN(\lambda ) \E \left[ \left( T_{K} + \lambda I \right)^{-\frac{1}{2}} T_{K,x_{i}} \left( T_{K} + \lambda I \right)^{-\frac{1}{2}} \right] \\
        & = E_{K}^{2} \mcN(\lambda ) (T_{K} + \lambda I )^{-1}T_{K} := V
    \end{aligned}
    \end{equation*}
    where $A\preceq B$ denotes $B-A$ is a positive semi-definite operator. Notice 
    \begin{equation*}
        \|V\|_{op} = \mcN(\lambda) \frac{\|T_{K}\|_{op}}{\|T_{K}\|_{op} + \lambda}  \leq \mcN(\lambda), \quad \text{and} \quad tr(V) = \mcN(\lambda)^2.
    \end{equation*}
    The proof is finished by applying Lemma E.3 in \citet{zhang2023optimality} to $A_{i}$ and $V$.
\end{proof}

\begin{proposition}\label{prop: bounds for A3}
    Suppose that Assumptions in the estimation error theorem hold. We have
    \begin{equation*}
         \left\| \left( T_{K} + \lambda I  \right)^{-\frac{1}{2}} \left( g_{n}  - \left( T_{K,n} + \lambda I  \right)^{-1}f_{\lambda} \right)  \right\|_{\mcH_{K}} \leq C ln(\frac{4}{\delta}) n^{-\frac{\alpha_{0}}{2\alpha_{0}+d}} 
    \end{equation*}
    where $C$ is a universal constant.
\end{proposition}
\begin{proof}
    Denote
    \begin{equation*}
    \begin{aligned}
        & \xi_{i} = \xi(x_i,y_i) = \left( T_{K} + \lambda I \right)^{-\frac{1}{2}}(K_{x_i}y_i - T_{K,x_i}f_{\lambda}) \\
        & \xi_{x} = \xi(x,y) = \left( T_{K} + \lambda I \right)^{-\frac{1}{2}}(K_{x}y - T_{K,x}f_{\lambda}),
    \end{aligned}
    \end{equation*}
    then it is equivalent to show 
    \begin{equation*}
        \left\| \frac{1}{n}\sum_{i=1}^{n} \xi_{i}  - \E \xi_{x} \right\|_{\mcH_{K}} \leq C ln(\frac{4}{\delta}) n^{-\frac{\alpha_{0}}{2\alpha_{0}+d}} 
    \end{equation*}
    Define $\Omega_{1} = \{x\in \mcX: |f_{0}|\leq t\}$ and $\Omega_{2} = \mcX\backslash\Omega_{1}$. We decomposite $\xi_{i}$ and $\xi_{x}$ over $\Omega_{1}$ and $\Omega_{2}$, which leads to 
    \begin{equation*}
    \begin{aligned}
        \left\| \left( T_{K} + \lambda I  \right)^{-\frac{1}{2}} \left( g_{n}  - \left( T_{K,n} + \lambda I  \right)^{-1}f_{\lambda} \right)  \right\|_{\mcH_{K}} & \leq \underbrace{\left\|\frac{1}{n} \sum_{i=1}^n \xi_i I_{x_i \in \Omega_1}-\E \xi_x I_{x \in \Omega_1}\right\|_{\mcH_{K}}}_{I_{1}} \\
        & \quad + \underbrace{\left\|\frac{1}{n} \sum_{i=1}^n \xi_i I_{x_i \in \Omega_2}\right\|_{\mcH_{K}}}_{I_{2}} + \underbrace{\left\|\E \xi_x I_{x \in \Omega_2}\right\|_{\mcH_{K}}}_{I_{3}}.
    \end{aligned}
    \end{equation*}
    For $I_{1}$, applying Proposition~\ref{prop: bounds on I1}, for any $\delta \in (0,1)$, with probability $1-\delta$, we have 
    \begin{equation}\label{eqn: upper bound for I1}
        I_{1} \leq \log (\frac{2}{\delta}) 
        \left( \frac{C_{1}\sqrt{\mcN(\lambda)}}{n} \tilde{M} + \frac{C_{2}\sqrt{\mcN(\lambda)}}{\sqrt{n}} + \frac{C_{1}log(\frac{1}{\lambda})^{-\frac{\alpha_{0}}{2}}\sqrt{\mcN(\lambda)}}{\sqrt{n}}  \right)
    \end{equation}
    where $C_{1} = 8\sqrt{2}$, $C_{2} = 8\sigma$ and $\tilde{M} = L + (\mcN(\lambda)+1)t$. By choosing $log(1/\lambda) \asymp \exp\{ -C n^{\frac{2}{2\alpha_{0}+d}}\}$ and applying Lemma~\ref{lemma: bound for effective dimension}, we have
    \begin{itemize}
        \item for the second term in (\ref{eqn: upper bound for I1}),
        \begin{equation*}
            \frac{C_{2}\sqrt{\mcN(\lambda)}}{\sqrt{n}} \asymp n^{-\frac{\alpha_{0}}{2\alpha_{0}+d}}.
        \end{equation*}

        \item for the third term, 
        \begin{equation*}
             \frac{C_{1}log(\frac{1}{\lambda})^{-\frac{\alpha_{0}}{2}}\sqrt{\mcN(\lambda)}}{\sqrt{n}} \lesssim  \frac{C_{2}\sqrt{\mcN(\lambda)}}{\sqrt{n}} \asymp n^{-\frac{\alpha_{0}}{2\alpha_{0}+d}}.
        \end{equation*}

        \item for the first term,
        \begin{equation*}
            \frac{C_{1}\sqrt{\mcN(\lambda)}}{n} \tilde{M}\leq \frac{C_1 L \sqrt{\mcN(\lambda)}}{n} + \frac{C_1 t \mcN(\lambda)^{\frac{3}{2}}}{n}  \lesssim n^{-\frac{\alpha_{0}}{2\alpha_{0}+d}} \quad \text{given} \quad t \leq n^{\frac{2\alpha_{0}-d}{2(2\alpha_{0}+d)}}.
        \end{equation*}
    \end{itemize}
    Combining all facts, if $t \leq n^{\frac{2\alpha_{0}-d}{2(2\alpha_{0}+d)}}$, with probability $1-\delta$ we have 
    \begin{equation*}
        I_{1} \leq C ln(\frac{2}{\delta}) n^{-\frac{\alpha_{0}}{2\alpha_{0}+d}}.
    \end{equation*}

    For $I_{2}$, we have 
    \begin{equation*}
    \begin{aligned}
    \tau_n:=P\left(I_{2}> \frac{\sqrt{\mcN(\lambda)}}{\sqrt{n}} \right) & \leq P\left(\exists x_i \text { s.t. } x_i \in \Omega_2\right) \\
    & =1-P\left(x \notin \Omega_2\right)^n \\
    & =1-P\left(\left|f_{0}(x)\right| \leq t\right)^n \\
    & \leq 1-\left(1-\frac{\left(C_q\right)^q}{t^q}\right)^n .
    \end{aligned}
    \end{equation*}
    Letting $\tau_{n} \rightarrow 0$ leading $t\gg n^{\frac{1}{q}}$. That is to say, if $t\gg n^{\frac{1}{q}}$ holds, we have $\tau_{n} = P\left(I_{2}> I_{1} \right) \rightarrow 0$.

    For $I_{3}$, we have 
    \begin{equation*}
    \begin{aligned}
        I_{3} & \leq \E \left\| \xi_{x}I_{x\in \Omega_{2}} \right\|_{\mcH_{K}} \\
        & \leq \E \left[ \left\| (T_{K} + \lambda I)^{-\frac{1}{2}} K(x,\cdot) \right\|_{\mcH_{K}} \left| \left( y - f_{0}(x) \right)I_{x\in \Omega_{2}} \right| \right] \\
        & \leq E_{K}^{2} \mcN(\lambda) \E \left| \left( y - f_{0}(x) \right)I_{x\in \Omega_{2}} \right| \\
        & \leq E_{K}^{2} \mcN(\lambda) \left( \E \left| \left( f_{\lambda} - f_{0}(x) \right)I_{x\in \Omega_{2}} \right| + \E \left| \epsilon I_{x\in \Omega_{2}} \right| \right)
    \end{aligned}
    \end{equation*}
    Using Cauchy-Schwarz inequality yields 
    \begin{equation*}
         \E \left| \left( f_{\lambda} - f_{0}(x) \right)I_{x\in \Omega_{2}} \right| \leq \left\|f_{0} - f_{\lambda} \right\|_{L_{2}} P(x\in \Omega_{2}) \leq log(\frac{1}{\lambda})^{-\frac{\alpha_{0}}{2}} (C_{q})^{q} t^{-q}
    \end{equation*}
    In addition, we have 
    \begin{equation*}
        \E \left| \epsilon I_{x\in \Omega_{2}} \right| \leq \sigma  \E \left|  I_{x\in \Omega_{2}} \right| \leq \sigma (C_{q})^{q} t^{-q}.
    \end{equation*}
    Together, we have 
    \begin{equation*}
        I_{3} \leq log(\frac{1}{\lambda})^{-\frac{\alpha_{0}}{2}} (C_{q})^{q} t^{-q} + \sigma (C_{q})^{q} t^{-q}.
    \end{equation*}
    Notice if we pick $q \geq \frac{2(2\alpha_{0}+d)}{2\alpha_{0}-d}$, there exist $t$ such that with probability $1-\delta - \tau_{n}$, we have 
    \begin{equation*}
        I_{1} + I_{2} + I_{3} \leq C ln(\frac{2}{\delta}) n^{-\frac{\alpha_{0}}{2\alpha_{0}+d}}. 
    \end{equation*}
    For fixed $\delta$, as $n\rightarrow\infty$, $\tau_{n}$ is sufficiently small such that $\tau_{n} = o(\delta)$, therefore without loss of generality, we can say with probability $1-\delta - \tau_{n}$, we have 
    \begin{equation*}
        I_{1} + I_{2} + I_{3} \leq C ln(\frac{2}{\delta}) n^{-\frac{\alpha_{0}}{2\alpha_{0}+d}}. 
    \end{equation*}
\end{proof}

\begin{proposition}\label{prop: bounds on I1}
    Under the same conditions as the Proposition, we have 
    \begin{equation*}
        \left\|\frac{1}{n} \sum_{i=1}^n \xi_i I_{x_i \in \Omega_1}-\E \xi_x I_{x \in \Omega_1}\right\|_{\mcH_{K}}\leq \log(\frac{2}{\delta}) 
        \left( \frac{C_{1}\sqrt{\mcN(\lambda)}}{n} \tilde{M} + \frac{C_{2}\sqrt{\mcN(\lambda)}}{\sqrt{n}} + \frac{C_{1}log(\frac{1}{\lambda})^{-\frac{\alpha_{0}}{2}}\sqrt{\mcN(\lambda)}}{\sqrt{n}}  \right)
    \end{equation*}
    where $C_{1} = 8\sqrt{2}$, $C_{2} = 8\sigma$ and $\tilde{M} = L + (\mcN(\lambda)+1)t$.
\end{proposition}

\begin{proof}
    In order to use Bernstein inequality (Lemma~\ref{lemma: Bernstein inequality}), we first bound the $m$-th moment of $\xi_{x}I_{x\in \Omega_{1}}$.
    \begin{equation*}
    \begin{aligned}
    \E\left\|\xi_{x} I_{x \in \Omega_{1}}\right\|_{\mcH_{K}}^m & =\E\left\|\left( T_{K} + \lambda I \right)^{-\frac{1}{2}} K_x\left(y-f_\lambda(x)\right) I_{x \in \Omega_{1}}\right\|_{\mcH_{K}}^m \\
    & \leq \E\left(\left\|\left( T_{K} + \lambda I \right)^{-\frac{1}{2}} K(x, \cdot)\right\|_{\mcH_{K}}^m \E\left(\left|\left(y-f_\lambda(x)\right) I_{x \in \Omega_{1}}\right|^m \mid x\right)\right) .
    \end{aligned}
    \end{equation*}
    Using the inequality $(a+b)^m \leq 2^{m-1}\left(a^m+b^m\right)$, we have
    \begin{equation*}
    \begin{aligned}
    \left|y-f_\lambda(x)\right|^m & \leq 2^{m-1}\left(\left|f_\lambda(x)-f_\rho^*(x)\right|^m+\left|f_\rho^*(x)-y\right|^m\right) \\
    & =2^{m-1}\left(\left|f_\lambda(x)-f_\rho^*(x)\right|^m+|\epsilon|^m\right) .
    \end{aligned}
    \end{equation*}
    Combining the inequalities, we have
    \begin{equation*}
    \begin{aligned}
     \E\left\|\xi_{x} I_{x \in \Omega_{1}}\right\|_{\mcH_{K}}^m & \leq \underbrace{2^{m-1} \E\left(\left\|\left( T_{K} + \lambda I \right)^{-\frac{1}{2}} K(x, \cdot)\right\|_{\mcH_{K}}^m\left|\left(f_\lambda(x)-f_\rho^*(x)\right) I_{x \in \Omega_{1}}\right|^m\right)}_{B_{1}} \\
    &\quad + \underbrace{2^{m-1} \E\left(\left\|\left( T_{K} + \lambda I \right)^{-\frac{1}{2}} K(x, \cdot)\right\|_{\mcH_{K}}^m \E\left(\left|\epsilon I_{x \in \Omega_{1}}\right|^m \mid x\right)\right)}_{B_{2}}.
    \end{aligned}
    \end{equation*}
    We first focus on $B_{2}$, by Lemma~\ref{lemma: bound for kernel function}, we have
    \begin{equation*}
        \E\left\|\left( T_{K} + \lambda I \right)^{-\frac{1}{2}} K(x, \cdot)\right\|_{\mcH_{K}}^m \leq \left( E_{K}^{2} \mcN(\lambda) \right)^{\frac{m}{2}}.
    \end{equation*}
    By the error moment assumption, we have 
    \begin{equation*}
        \E\left(\left|\epsilon I_{x \in \Omega_{1}}\right|^m \mid x\right) \leq \E\left(\left|\epsilon \right|^m \mid x\right) \leq \frac{1}{2}m! \sigma^2 L^{m-2},
    \end{equation*}
    together, we have 
    \begin{equation}\label{eqn: bound on A2}
        B_{2} \leq \frac{1}{2} m! \left(\sqrt{2}\sigma \sqrt{\mcN(\lambda)} \right)^{2} \left(2L \mcN(\lambda)\right)^{m-2}.
    \end{equation}
    Turning to bounding $B_{1}$, we first have 
    \begin{equation*}
    \begin{aligned}
        \left\| (f_{\lambda} - f_{0})I_{x\in \Omega_{1}} \right\|_{L_{\infty}} & \leq \left\| f_{\lambda} I_{x\in \Omega_{1}} \right\|_{L_{\infty}} + \left\|  f_{0} I_{x\in \Omega_{1}} \right\|_{L_{\infty}} \\
        & \leq \left\| (T_{K} + \lambda I)^{-1}T_{K}(f_{0}) I_{x\in \Omega_{1}} \right\|_{L_{\infty}} + \left\|  f_{0} I_{x\in \Omega_{1}} \right\|_{L_{\infty}} \\
        & \leq \left( \left\| (T_{K}+\lambda I)^{-1}T_{K} \right\|_{op} + 1  \right)\left\|  f_{0} I_{x\in \Omega_{1}} \right\|_{L_{\infty}}\\
        & \leq \left( \mcN(\lambda) + 1 \right) t := M.
    \end{aligned}
    \end{equation*}
    With bounds on approximation error, we get the upper bound for $B_{1}$ as
    \begin{equation}\label{eqn: bound on A1}
    \begin{aligned}
        B_{1} & \leq 2^{m-1} \mcN(\lambda)^{\frac{m}{2}}  \left\| (f_{\lambda} - f_{0})I_{x\in \Omega_{1}} \right\|_{L_{\infty}}^{m-2}  \left\| (f_{\lambda} - f_{0})I_{x\in \Omega_{1}} \right\|_{L_{2}}^2\\
        & \leq 2^{m-1} \mcN(\lambda)^{\frac{m}{2}} M^{m-2} log(\frac{1}{\lambda})^{-\alpha_{0}} \\
        & \leq \frac{1}{2}m! \left( 2 log(\frac{1}{\lambda})^{-\frac{\alpha_{0}}{2}} \sqrt{\mcN(\lambda)} \right)^2 \left( 2M \sqrt{\mcN(\lambda)}\right)^{m-2}. 
    \end{aligned}
    \end{equation}
    Denote 
    \begin{equation*}
        \begin{aligned}
            & \tilde{L} = 2(L+M)\sqrt{\mcN(\lambda)} \\
            & \tilde{\sigma} = \sqrt{2}\sigma \sqrt{\mcN(\lambda)} + 2 log(\frac{1}{\lambda})^{-\frac{\alpha_{0}}{2}} \sqrt{\mcN(\lambda)}
        \end{aligned}
    \end{equation*}
    and combine the upper bounds for $B_{1}$ and $B_{2}$, i.e. (\ref{eqn: bound on A1}) and (\ref{eqn: bound on A2}), then we have 
    \begin{equation*}
         \E\left\|\xi_{x} I_{x \in \Omega_{1}}\right\|_{\mcH_{K}}^m \leq \frac{1}{2} m! \tilde{\sigma}^2 \tilde{L}^{m-2}.
    \end{equation*}
    The proof is finished by applying Bernstein inequality i.e. Lemma~\ref{lemma: Bernstein inequality}.
\end{proof}

\subsubsection{Auxiliary Lemma}
\begin{lemma}\label{lemma: bound for effective dimension}
If $s_{j}  = C_{1} \exp(-C_{2} j^{2})$, by choosing $\log(1/\lambda)\asymp n^{\frac{2}{2\alpha_{0}+d}}$, we have 
\begin{equation*}
    \mcN(\lambda) = O\left( n^{\frac{d}{2\alpha_{0}+d}} \right)
\end{equation*}
\end{lemma}
\begin{proof}
For a positive integer $J\geq 1$
\begin{equation*}
\begin{aligned}
    \mcN(\lambda) & = \sum_{j=1}^{J} \frac{s_{j}}{s_{j} + \lambda} + \sum_{j = J+1}^{\infty} \frac{s_{j}}{s_{j} + \lambda} \\
    & \leq J + \sum_{j = J+1}^{\infty} \frac{s_{j}}{s_{j} + \lambda} \\
    & \leq J + \frac{C_{1}}{\lambda} \int_{J}^{\infty} exp\{ -C_{2}x^{2} \} dx \\
    & \leq J + \frac{1}{ \lambda} \frac{C_{1} exp\{-C_{2} J^{2}\}}{2 C_{2} J}
\end{aligned}
\end{equation*}
where we use the fact that the eigenvalue of the Gaussian kernel decays at an exponential rate, i.e. $s_{j} \leq C_{1} \exp\{ -C_{2} j^{2} \}$ and the inequality 
\begin{equation*}
    \int_{x}^{\infty} exp\{ \frac{-t^2}{2} \}dt \leq \int_{x}^{\infty} \frac{t}{x} exp\{ \frac{-t^2}{2} \}dt  \leq \frac{exp\{-\frac{x^2}{2}\}}{x}.
\end{equation*}
Then select $J = \lfloor n^{\frac{d}{2 \alpha_{0} + d}} \rfloor$ and $\lambda = exp\{ - C^{'} n^{\frac{2}{2 \alpha_{0} + d}}\}$ with $C'\leq C_{2}$ leads to 
\begin{equation*}
    \mcN(\lambda) = O\left( n^{\frac{d}{2\alpha_{0}+d}} \right).
\end{equation*}
\end{proof}

\begin{lemma}\label{lemma: bound for kernel function}
    For $\mu$-almost $x\in \mcX$, we have 
    \begin{equation*}
        \left\| \left( T_{K} + \lambda I \right)^{-\frac{1}{2}} K(x,\cdot) \right\|_{\mcH_{K}}^{2} \leq E_{K}^{2} \mcN(\lambda), \quad \text{and}\quad \E \left\| \left( T_{K} + \lambda I \right)^{-\frac{1}{2}} K(x,\cdot) \right\|_{\mcH_{K}}^{2} \leq \mcN(\lambda).
    \end{equation*}
    For some constant $E_K$. Consequently, we also have 
    \begin{equation*}
        \left\| \left( T_{K} + \lambda I \right)^{-\frac{1}{2}} T_{K,x} \left( T_{K} + \lambda I \right)^{-\frac{1}{2}}\right\|_{op} \leq E_{K}^{2} \mcN(\lambda).
    \end{equation*}
\end{lemma}

\begin{proof}
    We first state a fact on the Gaussian kernel. If $K$ is a Gaussian kernel function with fixed bandwidth, then there exists a constant $E_{K}$ such that the eigenfunction of $K$ is uniformly bounded for all $j \geq 1$, i.e. $\sup_{j\geq 1}\|e_{j}\|_{L_{\infty}}\leq E_{K}$. This is indeed the so-called ``uniformly bounded eigenfunction'' assumption that usually appears in nonparametric regression literature, especially for those who consider misspecified kernel in KRR, see \citet{mendelson2010regularization,wang2022gaussian}. Based on the explicit construction of the RKHS associated with the Gaussian kernel \citep{steinwart2006explicit}, we know the uniformly bounded eigenfunction holds for the Gaussian kernel. 
    
    Based on the fact of uniformly bounded eigenfunction, we know $e_{j}^2(x) \leq {E_{K}^{2}}$ for all $x\in \mcX$ and $ j \geq 1$. Then, we prove the first inequality by the following procedure,
    \begin{equation*}
    \begin{aligned}
        \left\| \left( T_{K} + \lambda I \right)^{-\frac{1}{2}} K(x,\cdot) \right\|_{\mcH_{K}}^{2} & =  \left\| \sum_{j=1} \frac{1}{\sqrt{s_{j}+\lambda}} s_{j} e_{j}(x)e_{j}(\cdot)\right\|_{\mcH_{K}}^{2}\\
        & = \sum_{j=1}^{\infty} \frac{s_{j}}{s_{j} + \lambda} e_{j}^{2}(x)\\
        & \leq E_{K}^{2} \mcN(\lambda).
    \end{aligned}
    \end{equation*}
    The second inequality follows given the fact that $\E e_{j}^{2}(x) = 1$. The third inequality comes from the observation that for any $f\in \mcH_{K}$
    \begin{equation*}
         \left( T_{K} + \lambda I \right)^{-\frac{1}{2}} T_{K,x} \left( T_{K} + \lambda I \right)^{-\frac{1}{2}} (f) = \left \langle  \left( T_{K} + \lambda I \right)^{-\frac{1}{2}} K(x,\cdot), f \right\rangle_{\mcH_{K}}  \left( T_{K} + \lambda I \right)^{-\frac{1}{2}} K(x,\cdot)
    \end{equation*}
    and
    \begin{equation*}
    \begin{aligned}
        \left\| \left( T_{K} + \lambda I \right)^{-\frac{1}{2}} T_{K,x} \left( T_{K} + \lambda I \right)^{-\frac{1}{2}}\right\|_{op} & = \sup_{\|f\|_{\mcH_{k}}=1}\|\left( T_{K} + \lambda I \right)^{-\frac{1}{2}} T_{K,x} \left( T_{K} + \lambda I \right)^{-\frac{1}{2}} (f)\|_{\mcH_{K}} \\
         & = \left\| \left( T_{K} + \lambda I \right)^{-\frac{1}{2}} K(x,\cdot) \right\|_{\mcH_{K}}^{2}
    \end{aligned}
    \end{equation*}
\end{proof}

\begin{lemma}(Bernstein inequality)\label{lemma: Bernstein inequality}
Let $(\Omega, \mathcal{B}, P)$ be a probability space, $H$ be a separable Hilbert space, and $\xi: \Omega \rightarrow H$ be a random variable with
\begin{equation*}
    \E\|\xi\|_H^m \leq \frac{1}{2} m ! \sigma^2 L^{m-2}
\end{equation*}
for all $m>2$. Then for $\delta \in(0,1)$, $\xi_i$ are i.i.d. random variables, with probability at least $1-\delta$, we have
\begin{equation*}
    \left\|\frac{1}{n} \sum_{i=1}^n \xi_i-\E \xi\right\|_H \leq 4 \sqrt{2} \log \left( \frac{2}{\delta}\right)\left(\frac{L}{n}+\frac{\sigma}{\sqrt{n}}\right)
\end{equation*}
\end{lemma}

\subsection{Proof of Adaptive Rate (Theorem~\ref{thm: adaptive rate of single-task GKRR})}\label{apd: proof of adaptive rate}

\begin{proof}
    First, we show that it is sufficient to consider the true Sobolev space $\alpha$ in $\mcA = \{\alpha_{1},\cdots, \alpha_{N}\}$ with $ \alpha_{j} - \alpha_{j-1}\asymp 1/\logn$. If $\alpha_{0} \in (\alpha_{j-1}, \alpha_{j})$, then $H^{\alpha_{j}} \subset H^{\alpha_{0}} \subset H^{\alpha_{j-1}}$. Therefore, since $\psi_{n}(\alpha_{0})$ is squeezed between $\psi_{n}(\alpha_{j-1})$ and $\psi_{n}(\alpha_{j})$, we just need to show $\psi_{n}(\alpha_{j-1}) \asymp \psi_{n}(\alpha_{j})$. By the definition of $\psi_{n}(\alpha)$, the claim follows since
    \begin{equation*}
        \log \frac{\psi_{n}(\alpha_{j-1})}{\psi_{n}(\alpha_{j})} = \left( -\frac{2\alpha_{j-1}}{2\alpha_{j-1}+d} + \frac{2\alpha_{j}}{2\alpha_{j} + d} \right) \log \frac{n}{\logn} \asymp (\alpha_{j} - \alpha_{j-1})\logn \asymp 1.
    \end{equation*}
    Therefore, we assume $f_{0} \in H^{\alpha_{i}}$ where $i\in \{1,2,\cdots,N\}$. 

    Let $m = \lfloor \frac{n}{2} + 1 \rfloor$, i.e. $m\geq \frac{n}{2}$, by Theorem~\ref{thm: non-adaptive rate of KRR with Gaussian}, for some universal constant $C$ that doesn't depend on $n$, we have
    \begin{equation}\label{eqn: eqn1 in adaptive estimator via TV}
    \begin{aligned}
        \mcE(\hat{f}_{\lambda_{\alpha},\mcD_{1}}) & \leq \left( \log\frac{4}{\delta} \right)^2 \left(  \text{E}(\lambda_{\alpha},m) + \text{A}(\lambda_{\alpha},m) \right)\\
    \end{aligned}
    \end{equation}
    for all $\alpha \in \mcA$ simultaneously with probability at least $1 - N \delta$. Here, $\text{E}(\lambda,n)$ and $\text{A}(\lambda,n)$ denote the estimation and approximation error that depends on the regularization parameter $\lambda$ and sample size $n$ in non-adaptive rate proof. 
    
    Furthermore, by Theorem 7.2 in \citet{steinwart2008support} and Assumption~\ref{assump: error tail}, we have 
    \begin{equation}\label{eqn: eqn2 in adaptive estimator via TV}
    \begin{aligned}
        \mcE(\hat{f}_{\lambda_{\hat{\alpha}}}) & < 6 \left( \inf_{\alpha \in \mcA}\mcE(\hat{f}_{\lambda_{\alpha}}) \right) + \frac{128\sigma^{2}L^2 \left(\log\frac{1}{\delta} + \log(1 + N)\right)}{n-m}\\
        & < 6 \left( \inf_{\alpha \in \mcA}\mcE(\hat{f}_{\lambda_{\alpha}}) \right) + \frac{512 \sigma^{2}L^2 \left(\log\frac{1}{\delta} + \log(1 + N)\right)}{n}
    \end{aligned}
    \end{equation}
    with probability $1 - \delta$, where the last inequality is based on the fact that $n-m \geq \frac{n}{2} - 1 \geq \frac{n}{4}$.
    
    Combining (\ref{eqn: eqn1 in adaptive estimator via TV}) and (\ref{eqn: eqn2 in adaptive estimator via TV}), we have 
    \begin{equation*}
    \begin{aligned}
        \mcE(\hat{f}_{\lambda_{\hat{\alpha}}}) 
        & < 6 \left( \log\frac{4}{\delta} \right)^2 \left(  \inf_{\alpha \in \mcA}  \text{E}(\lambda_{\alpha},m) + \text{A}(\lambda_{\alpha},m) \right) + \frac{512 \sigma^{2}L^2 \left(\log\frac{1}{\delta} + \log(1 + N)\right)}{n} \\
        & \leq 6 C \left( \log\frac{4}{\delta} \right)^2 m^{-\frac{2\alpha_{0}}{2\alpha_{0} + d}} + \frac{512 \sigma^{2}L^2 \left(\log\frac{1}{\delta} + \log(1 + N)\right)}{n} \\
        & \leq 12 C \left( \log\frac{4}{\delta} \right)^2 n^{-\frac{2\alpha_{0}}{2\alpha_{0} + d}} + \frac{512 \sigma^{2}L^2 \left(\log\frac{1}{\delta} + \log(1 + N)\right)}{n} \\
    \end{aligned}
    \end{equation*}
    with probability at least $1 - (1+N)\delta$. With a variable transformation, we have 
    \begin{equation}\label{eqn: eqn3 in adaptive estimator via TV}
        \mcE(\hat{f}_{\lambda_{\hat{\alpha}}}) 
         \leq 12 C \left( \log\frac{4(1+N)}{\delta} \right)^2 n^{-\frac{2\alpha_{0}}{2\alpha_{0} + d}} + \frac{512 \sigma^{2}L^2 \left(\log\frac{1+N}{\delta} + \log(1 + N)\right)}{n} 
    \end{equation}
    with probability $1-\delta$. Therefore, for the first term 
    \begin{equation}\label{eqn: eqn4 in adaptive estimator via TV}
    \begin{aligned}
        12 C \left( \log\frac{4(1+N)}{\delta} \right)^2 n^{-\frac{2\alpha_{0}}{2\alpha_{0} + d}} & \leq 24C \left\{\left(\log\frac{4}{\delta} \right)^2 \log^{2}(1+N) + 1 \right\}n^{-\frac{2\alpha_{0}}{2\alpha_{0} + d}} \\
        & \leq 24C' \left(\log\frac{4}{\delta} \right)^2  \left( \frac{n}{\logn}\right)^{-\frac{2\alpha_{0}}{2\alpha_{0} + d}} + 24 C n^{-\frac{2\alpha_{0}}{2\alpha_{0} + d}}
    \end{aligned}
    \end{equation}
    where the first inequality is based on the fact that $a+b<ab+1$ for $a,b>1$, while the second inequality is based on the fact that $\log(x) \leq x^{\frac{\alpha_{0}}{2\alpha_{0} + d}}$ for some $n$ such that $\log(\logn) / \logn < 1/4$. For the second term,
    \begin{equation}\label{eqn: eqn5 in adaptive estimator via TV}
    \begin{aligned}
        \frac{512 \sigma^{2}L^2 \left(\log\frac{1+N}{\delta} + \log(1 + N)\right)}{n} & \leq \frac{512 \sigma^{2}L^2 \left(\log\frac{1}{\delta} + 1 + 2\log(1 + N)\right)}{n} \\
        &\leq \frac{512 \sigma^{2}L^2 \left(\log\frac{1}{\delta} + 1 + 2\logn\right)}{n}
    \end{aligned}
    \end{equation}
    The proof is finished by combining (\ref{eqn: eqn3 in adaptive estimator via TV}), (\ref{eqn: eqn4 in adaptive estimator via TV}) and (\ref{eqn: eqn5 in adaptive estimator via TV}).
\end{proof}

\section{Smoothness Adaptive Transfer Learning Results}

\subsection{Proof of Lower Bound}\label{apd: proof of lower bound of SATL}
In this part, we proof the alternative version for the lower bound, i.e. 
\begin{equation*}
    \inf_{\tilde{f}} \sup_{(P_{T},P_{S})} \E_{\rho_{T}} \| \tilde{f} - f_{T} \|_{L_{2}}^2 \geq C \left( (n_{S} + n_{T})^{-\frac{2m_{0}}{2m_{0}+d}} + n_{T}^{-\frac{m}{2m+d}}\xi(h,f_{S}) \right).
\end{equation*}
for some universal constant $C$. This alternative form is also used to prove the lower bound in other transfer learning contexts like high-dimensional linear regression or GLM, see \citet{li2022transfer,tian2022transfer}. However, the upper bound we derive for SATL can still be sharp since in the transfer learning regime, it is always assumed $n_{S} \gg n_{T}$, and leads to $(n_{S} + n_{T})^{-\frac{2m_{0}}{2m_{0}+1}} \asymp n_{S}^{-\frac{2m_{0}}{2m_{0}+1}}$. 

On the other hand, one can modify the first phase in OTL by including the target dataset to obtain $\hat{f}_{S}$, which produces an alternative upper bound $(n_{S} + n_{T})^{-\frac{2m_{0}}{2m_{0}+1}} + n_{T}^{-\frac{2m}{2m+1}} \xi(h,f_{S})$, and mathematically aligns with the alternative lower bound we mention above. However, we would like to note that such a modified OTL is not computationally efficient for transfer learning, since for each new upcoming target task, OTL needs to recalculate a new $\hat{f}_{S}$ with the combination of the target and source datasets.

Note that any lower bound for a specific case will immediately yield a lower bound for the general case. Therefore, we consider the following two cases.

(1) Consider $h=0$, i.e. both source and target data are drawn from $\rho_{T}$. In this case, the problem can be viewed as obtaining the lower bound for classical nonparametric regression with sample size $n_{T} + n_{S}$ and prediction function as $f_{T}$. Then using the Proposition~\ref{prop: Lower bound for target-only KRR}, we have
\begin{equation*}
    \inf_{\tilde{f}} \sup_{(\rho_{T},\rho_{S})} \E \| \tilde{f} - f_{T} \|_{L_{2}}^2 \geq C \| f_{T} \|_{H^{m_{0}}}^2   (n_{T} +n_{S})^{-\frac{2m_{0}}{2m_{0} + d}}.
\end{equation*}

(2) Consider $f_{S} = 0$, i.e. the source model has no similarity to $f_{T}$ and all the information about $f_{T}$ is stored in the target dataset. By the assumptions, we have $f_{T}  \in \left\{ f : f \in H^{m}, \|f\|_{H^{m}}\leq h  \right\}$. Again, using the Proposition~\ref{prop: Lower bound for target-only KRR}, we have
\begin{equation*}
    \inf_{\tilde{f}} \sup_{(\rho_{T},\rho_{S})} \E \| \tilde{f} - f_{T} \|_{L_{2}}^2 \geq  C  h^2  (n_{T} )^{-\frac{2m}{2m + d}}.
\end{equation*}
Combining the lower bound in case (1) and case (2), and by slightly adjusting the constant terms, we obtain the desired lower bound.

\subsection{Proof of Upper Bound}\label{apd: proof of upper bound of SATL}
The minimization problem in phase $2$ is
\begin{equation*}
    \hat{f}_{\delta} = \underset{f_{\delta} \in \mcH_{K_{m}}}{\operatorname{argmin}} \left \{ \frac{1}{n_{T}} \sum_{i=1}^{n_{T}} \left (y_{T,i} - \hat{f}_{S}(x_{T,i}) - f_{\delta}(x_{T,i}) \right)^2 + \lambda \|f_{\delta}\|_{\mcH_{K_{m}}}^2 \right\} .
\end{equation*}
The final estimator for target regression function is $\hat{f}_{T} = \hat{f}_{S} + \hat{f}_{\delta}$. By triangle inequality
\begin{equation}\label{eqn: triangle for upper bound}
    \left\| \hat{f}_{T} - f_{T}\right\|_{L_{2}} \leq \left\| \hat{f}_{S} - f_{S}\right\|_{L_{2}} + \left\| \hat{f}_{\delta} - f_{\delta}\right\|_{L_{2}}
\end{equation}
For the first term in the r.h.s. of (\ref{eqn: triangle for upper bound}), since the marginal distribution over $\mcX$ are equivalent for both target and source, applying Theorem 1 directly leads to with probability at least $1-\delta$
\begin{equation*}
    \left\| \hat{f}_{S} - f_{S}\right\|_{L_{2}}  \leq C \left(\log \frac{4}{\delta}\right)  \left(   n_{S}^{-\frac{m_{0}}{2m_{0} + d}} \right),
\end{equation*}
where $C$ is independent of $n_{S}$ and $\delta$, and proportional to $\sqrt{\sigma_{S}^2 +\|f_{S}\|_{H^{m_{0}}}^2}$. Notice the fact that $\frac{\sigma_{S}^2}{\|f_{\mcS}\|_{K}^2}$ is bounded above given the bounded assumption on $f_{S}$ and $\sigma^{2}$ (this is also a reasonable condition since the signal-to-noise ratio can't be 0, otherwise one can hardly learn anything from the data), we get $C\lesssim \|f_{S}\|_{H^{m_{0}}} $.

For the second term, we modify the proof of Theorem 1 with the same logic. Note that the offset function has the following expression
\begin{equation*}
    \hat{f}_{\delta} =  \left( T_{K,n} + \lambda_{2} \mathbf{I} \right)^{-1} \left( \frac{1}{n_{T}}\sum_{i=1}^{n_{T}} K_{x_{T,i}} \left( y_{T,i} - \hat{f}_{\mcS}(x_{T,i}) \right) \right)
\end{equation*}
Similarly, we introduce the intermediate term $f_{\delta,\lambda}$ as
\begin{equation*}
    f_{\delta, \lambda} = \left( T_{K} + \lambda_{2} \mathbf{I} \right)^{-1} \left( T_{K}(f_{\delta}) \right)
\end{equation*}
To control the approximation error, using Proposition~\ref{prop: bounds for approximation error} with the fact that $\|f_{\delta}\|_{H^{m}}\leq h$, we have 
\begin{equation*}
    \left\| f_{\delta,\lambda} - f_{\delta} \right\|_{L_{2}}^2 \leq log(\frac{1}{\lambda_{2}})^{-m} h^2.
\end{equation*}
For the estimation error, notice by change the $y_i$ and $y$ in Proposition~\ref{prop: bounds for approximation error} to $y_{T,i} - \hat{f}_{\mcS}(x_{T,i})$ and $y_{T} - \hat{f}_{\mcS}(x_{T})$ will not change the conclusion. Therefore, the same proof procedure of Theorem~\ref{thm: bounds for estimation error} can be directly applied on bounding the estimation error of $\hat{f}_{\delta}$, i.e. 
\begin{equation*}
    \left\| \hat{f}_{\delta} - f_{\delta,\lambda} \right\|_{L_{2}} \leq C ln(\frac{4}{\delta}) n^{-\frac{m}{2m+d}} 
\end{equation*}
holds with probability $1-\delta$. Combining the approximation and estimation error, we have, with probability $1-\delta$,
\begin{equation*}
    \left\| \hat{f}_{\delta} - f_{\delta} \right\|_{L_{2}} \leq C \left( \log\frac{4}{\delta} \right) \left(  n_{T}^{-\frac{m}{2m+d}} \right)
\end{equation*}
with $C$ independent of $n_{T}$ and $\delta$, and $C$ is proportional to $\sqrt{\sigma_{T}^2 + h^2} \cdot$. Notice the fact that $\frac{\sigma_{T}^2}{h^2}$ is bounded above (similar reason as phase 1), we get $C\lesssim h $.

Finally, the proof is finished by combining the result of $\hat{f}_{S}$ and $\hat{f}_{\delta}$, modifying the constant term and noticing the results holds with probability at least $(1-\delta)\cdot(1-\delta) \geq 1 -2\delta$.

\subsection{Propositions}
\begin{proposition}[Lower bound for target-only KRR]\label{prop: Lower bound for target-only KRR}
    In target-only KRR problem, suppose the observed data are $\{ (x_{i},y_{i}) \}_{i=1}^{n}$ and the underlying function $f \in H^{m_{0}}$, Then  
    \begin{equation*}
        \inf_{\tilde{f}} \sup_{P} \E \| \tilde{f} - f \|_{L_{2}}^2 = \Omega \left( n^{-\frac{2m_{0}}{2m_{0} + d}} \right)
    \end{equation*}
\end{proposition}
\begin{remark}
    The proof for the lower bound in target-only KRR is standard and can be found in many literatures. 
\end{remark}

\begin{proof}
    Consider a series functions $f_{1},\cdots,f_{M} \in H^{m_{0}}$ with norms equal to $\|f\|_{H^{m_{0}}}$. We construct such a series by 
    \begin{equation*}
        f_{i} = \sum_{k=N+1}^{2N} \frac{f_{i,k-N}  \|f\|_{H^{m_{0}}}}{\sqrt{N}} L_{K^{\frac{1}{2}}} (\phi_{k})
    \end{equation*}
    where $\{\phi_{k}\}_{k\geq 1}$ are the a basis of $L_{2}$, $K$ is the reproducing kernel of $H^{m_{0}}$ and $L_{K^{\frac{1}{2}}}$ is the integral operator of $K^{\frac{1}{2}}$. Then one can see the constructed $f_{i} \in H^{m_{0}}$ with $\|f_{i}\|_{H^{m_{0}}} = \|f\|_{H^{m_{0}}}$. 

    Let $P_{1}, \cdots, P_{M}$ be the probability distributions corresponding to $f_{1},\cdots, f_{M}$. Note that any lower bound for a specific case leads to a lower bound for the general case. Hence, we consider the error term $\epsilon_{i}$ are drawn from a Gaussian distribution, i.e. $\epsilon_{i} \sim N(0,\sigma^2)$. First note that 
    \begin{equation*}
        log( P_{i} / P_{j} ) = \frac{1}{\sigma^{2}} \sum_{k=1}^{n} \left(y_{k} - f_{j}(x_{k})\right) \left( f_{j}(x_{k}) - f_{i}(x_{k}) \right) - \frac{1}{2\sigma^2} \sum_{k=1}^{n} (f_{j}(x_{k}) - f_{i}(x_{k}))^2
    \end{equation*}
    Then the KL-divergence between $P_{i}$ and $P_{j}$ is 
    \begin{equation*}
        KL( P_{i} | P_{j} ) = \frac{n}{2\sigma^2} \| f_{i} - f_{j} \|_{L_{2}}^2. 
    \end{equation*}
    Let $\tilde{f}$ be any estimator based on observed data, and consider testing multiple hypotheses, by Markov inequality and Fano's Lemma,
    \begin{equation*}
    \begin{aligned}
        \E_{P_{i}}\|\tilde{f} - f_{i}\|_{L_{2}}^{2} & \geq P_{i}( \tilde{f} \neq f_{i}) \min_{i,j} \|f_{i} - f_{j}\|_{L_{2}}^2 \\
        & \geq \left( 1 - \frac{\frac{n}{2\sigma^2} \max_{i,j} \|f_{i} - f_{j}\|_{L_{2}}^2 + log(2)}{log(M-1)} \right) \min_{i,j} \|f_{i} - f_{j}\|_{L_{2}}^2.
    \end{aligned}
    \end{equation*}
    Notice that $s_{k} \asymp k^{-\frac{2m_{0}}{d}}$, then for any $i,j\in \{1,\cdots,M\}$
    \begin{equation*}
    \begin{aligned}
        \|f_{i} - f_{j}\|_{L_{2}}^{2} & = \frac{\|f\|_{H^{m_{0}}}^2}{N} \sum_{k=N+1}^{2N} (f_{i,k-N} - f_{j,k-N} )^2 s_{k} \\
        & \geq \frac{\|f\|_{H^{m_{0}}}^2 s_{2N}}{N} \sum_{k=N+1}^{2N} (f_{i,k-N} - f_{j,k-N} )^2 \\
        & \geq \frac{\|f\|_{H^{m_{0}}}^2 s_{2N}}{4}
    \end{aligned}
    \end{equation*}
    where the last inequality is based on Varshamov-Gilbert bound \citep{varshamov1957estimate} and 
    \begin{equation*}
    \begin{aligned}
        \|f_{i} - f_{j}\|_{L_{2}}^{2} & = \frac{\|f\|_{H^{m_{0}}}^2}{N} \sum_{k=N+1}^{2N} (f_{i,k-N} - f_{j,k-N} )^2 s_{k} \\
        & \leq \frac{\|f\|_{H^{m_{0}}}^2 s_{N}}{N} \sum_{k=N+1}^{2N} (f_{i,k-N} - f_{j,k-N} )^2 \\
        & \leq \|f\|_{H^{m_{0}}}^2 s_{N}
    \end{aligned}
    \end{equation*}
    Therefore, 
    \begin{equation*}
        \|\tilde{f} - f_{i}\|_{L_{2}}^{2} \geq \left(1 - \frac{\frac{4n \|f\|_{H^{m_{0}}}^2 }{\sigma^2} s_{N} + 8log(2) }{N}\right) \frac{ s_{2N} \|f\|_{H^{m_{0}}}^2 }{4}
    \end{equation*}
    and by taking $N = \left( \frac{8\|f\|_{H^{m_{0}}}^2}{\sigma^2}  \right)^{\frac{d}{2m_{0}}} n^{\frac{1}{\frac{2m_{0}}{d}+1}}$ will lead to 
    \begin{equation*}
        \|\tilde{f} - f_{i}\|_{L_{2}}^{2} \geq \left(1 - \frac{\frac{4n \|f\|_{H^{m_{0}}}^2 }{\sigma^2} s_{N} + 8log(2) }{N}\right) \frac{ s_{2N} \|f\|_{H^{m_{0}}}^2 }{4} \asymp n^{-\frac{2m_{0}}{2m_{0}+d}} \|f\|_{H^{m_{0}}}^2.
    \end{equation*}
\end{proof}

\section{Adaptive Rate via Lepski's Method}\label{apd: Lepski's Method}

In this section, we aim to leverage Lepski's method \citep{lepskii1991problem} to develop an adaptive estimator for the misspecified Gaussian kernel without knowing true smoothness. The adaptive procedure is summarised as follows.

Let $\hat{f}_{\lambda}$ be the estimator from spectral algorithms with regularization parameter $\lambda$. Suppose the unknown true Sobolev smoothness $\alpha_{0} \in [\alpha_{\min}, \alpha_{\max}]$ where $\alpha_{\min} > \frac{d}{2}$ and $\alpha_{\max}$ large enough. Define a discrete subset
\begin{equation*}
    \mcA = \left\{\alpha_{\min} = \alpha_{1} < \alpha_{2}< \cdots < \alpha_{N} = \alpha_{\max} \right\}
\end{equation*}
where $\alpha_{j} - \alpha_{j-1} \asymp (\log n)^{-1}$. We define the Lepski's estimator as 
\begin{equation*}\label{eqn: Lepski process}
    \hat{f}^{*} := \hat{f}_{\lambda_{\hat{\alpha}}}\quad  \text{with} \quad \hat{\alpha} = \max \left\{ \tilde{\alpha} \in \mcA: \|\hat{f}_{\lambda_{\tilde{\alpha}}} - \hat{f}_{\lambda_{\tilde{\alpha}'}}\|_{L_{2}(\mu)} \leq c_{0} \left(\frac{n}{\log(n)}\right)^{-\frac{\tilde{\alpha}'}{2\tilde{\alpha}'+d}}, \forall \tilde{\alpha}'\leq \tilde{\alpha}, \tilde{\alpha}'\in \mcA \right\}.
\end{equation*}
where $c_{0}$ are some finite constant. 

The following result shows that Lepski's estimator $\hat{f}^{*}$ achieves the optimal excess risk with Gaussian kernels up to a logarithmic factor in $n$, which is the price paid for adaptivity.
\begin{theorem}\label{thm: adaptive rate}
    Assume the true Sobolev smoothness $\alpha_{0}$ can be captured by $[\alpha_{\min}, \alpha_{\max}]$. Then under the same setting of Theorem~\ref{thm: non-adaptive rate of KRR with Gaussian}, for $s > \log(4)$, when $n$ is sufficient large, we have 
    \begin{equation*}
        \mcE(\hat{f}^{*}) \leq s^{2} \left(\frac{n}{\log(n)}\right)^{-\frac{2\alpha_{0}}{2\alpha_{0} + d}}
    \end{equation*}
    with probability at least $1 - 4\exp\{-s (\log(n))^{\frac{\alpha_{0}}{2\alpha_{0}+d}} \} - 4\exp\{-\frac{c_{0}}{2} (\log(n))^{\frac{\alpha_{\min}}{2\alpha_{\min}+d}} \}(\log(n))^2$.
\end{theorem}
\begin{remark}
When $\mu$ is unknown, one can get around the problem by using the balancing principle \citep{de2010adaptive}. This approach controls the empirical squared and the RKHS norm between two non-adaptive estimators in (\ref{eqn: Lepski process}) and the results generated from these alternative norms are combined to produce an adaptive estimator with the desired excess risk.
\end{remark}
\begin{proof}
Modify Theorem \ref{thm: non-adaptive rate of KRR with Gaussian} a bit, we get 
\begin{equation*}
    \mbP\left( \|\hat{f}_{\lambda_{\alpha_{0}}} - f_{0} \|_{L_{2}}^2 n^{\frac{2\alpha_{0}}{2\alpha_{0} + d}} \gtrsim s^{2} \right) \leq 4\exp(-s), \quad \forall s \geq \log(4),
\end{equation*}
and then
\begin{equation}\label{eqn: prob inequality with logn}
    \mbP\left( \|\hat{f}_{\lambda_{\alpha_{0}}} - f_{0} \|_{L_{2}}^2 \left(\frac{n}{\logn}\right)^{\frac{2\alpha_{0}}{2\alpha_{0} + d}} \gtrsim s^{2} \right) \leq 4\exp\left(-s (\logn)^{\frac{\alpha_{0}}{2\alpha_{0} + d}} \right).
\end{equation}
To simplify the notation, we define $\lambdat = \log(\frac{1}{\lambda})^{-\frac{1}{2}}$ and 
\begin{equation*}
    \lambdat_{\alpha} = \left( \frac{n}{\logn} \right)^{-\frac{1}{2\alpha + d}} \quad \text{and} \quad \psi_{n}(\alpha) = \left(\lambdat_{\alpha}\right)^{2\alpha} = \left( \frac{n}{\logn} \right)^{-\frac{2\alpha}{2\alpha + d}}.
\end{equation*}
Following these definitions, the Lepski's estimator can be reformulated as 
\begin{equation*}
    \hat{f}^{*} := \hat{f}_{\lambdat_{\hat{\alpha}}}\quad  \text{with} \quad \hat{\alpha} = \max \left\{ \tilde{\alpha} \in \mcA: \|\hat{f}_{\lambdat_{\tilde{\alpha}}} - \hat{f}_{\lambdat_{\tilde{\alpha}'}}\|_{L_{2}} \leq c_{0} \left(\frac{n}{\log(n)}\right)^{-\frac{\tilde{\alpha}'}{2\tilde{\alpha}'+d}}, \forall \tilde{\alpha}'\leq \tilde{\alpha}, \tilde{\alpha}'\in \mcA \right\}.
\end{equation*}
We remark here that since $\lambdat$ one-to-one corresponds to $\lambda$, the definition of Lepski's estimator $\hat{f}^{*}$ doesn't change.
Next, we aim to prove 
\begin{equation*}
    \sup_{\tilde{\alpha} \in \mcA} \sup_{f_{0}\in H^{\alpha}} \frac{\left\| \hat{f}^{*} - f_{0} \right\|_{L_{2}}^{2}}{\psi_{n}(\alpha_{0})}  \lesssim s^2 
\end{equation*}
with high-probability.

First, we show that it is sufficient to consider the true Sobolev space $\alpha_{0}$ in $\mcA$. If $\alpha_{0} \in (\alpha_{j-1}, \alpha_{j})$, then $H^{\alpha_{j}} \subset H^{\alpha_{0}} \subset H^{\alpha_{j-1}}$. Therefore, since $\psi_{n}(\alpha_{0})$ is squeezed between $\psi_{n}(\alpha_{j-1})$ and $\psi_{n}(\alpha_{j})$, we just need to show $\psi_{n}(\alpha_{j-1}) \asymp \psi_{n}(\alpha_{j})$. By the definition of $\psi_{n}(\alpha)$, the claim follows since
\begin{equation*}
    \log \frac{\psi_{n}(\alpha_{j-1})}{\psi_{n}(\alpha_{j})} = \left( -\frac{2\alpha_{j-1}}{2\alpha_{j-1}+d} + \frac{2\alpha_{j}}{2\alpha_{j} + d} \right) \log \frac{n}{\logn} \asymp (\alpha_{j} - \alpha_{j-1})\logn \asymp 1.
\end{equation*}
Therefore, we assume $f_{0} \in H^{\alpha_{i}}$ where $i\in \{1,2,\cdots,N\}$. Let $\mcE_{j}$ be the event that $\hat{\alpha} = \alpha_{j}$. We have 
\begin{equation*}
    \left\| \hat{f}^{*} - f_{0} \right\|_{L_{2}}^2 \psi_{n}(\alpha_{i})^{-1} = \sum_{j=1}^{N} \left\| \hat{f}_{\lambdat_{\alpha_{j}}} - f_{0} \right\|_{L_{2}}^2  \psi_{n}(\alpha_{i})^{-1} \mathbf{1}_{\mcE_{j}}.
\end{equation*}
For $j \geq i$, on the event $\mcE_{j}$, it holds that 
\begin{equation*}
    \left\|\hat{f}_{\lambdat_{\alpha_{j}}}- \hat{f}_{\lambdat_{\alpha_{i}}} \right\|_{L_{2}}^2  \psi_{n}(\alpha_{i})^{-1} \leq c_{0}^2
\end{equation*}
by the definition of the estimator. Hence
\begin{equation*}
\begin{aligned}
    &\sum_{j=i}^{N} \left\| \hat{f}_{\lambdat_{\alpha_{j}}} - f_{0} \right\|_{L_{2}}^2  \psi_{n}(\alpha_{i})^{-1} \mathbf{1}_{\mcE_{j}}\\
    = & \sum_{j=i}^{N} \left\| \hat{f}_{\lambdat_{\alpha_{j}}} - \hat{f}_{\lambdat_{\alpha_{i}}} + \hat{f}_{\lambdat_{\alpha_{i}}}- f_{0} \right\|_{L_{2}}^2  \psi_{n}(\alpha_{i})^{-1} \mathbf{1}_{\mcE_{j}}\\
    \leq & \sum_{j=i}^{N} \left\{2c_{0}^2 + 2 \left\|  \hat{f}_{\lambdat_{\alpha_{i}}}- f_{0} \right\|_{L_{2}}^2  \psi_{n}(\alpha_{i})^{-1}  \right\}  \mathbf{1}_{\mcE_{j}}\\
    \leq & 2c_{0}^2 + \left\|\hat{f}_{\lambdat_{\alpha_{i}}}- f_{0} \right\|_{L_{2}}^2  \psi_{n}(\alpha_{i})^{-1}.
\end{aligned}
\end{equation*}
Thus, following the result from (\ref{eqn: prob inequality with logn}), we have
\begin{equation*}
    \sum_{j=i}^{N} \left\| \hat{f}_{\lambdat_{\alpha_{j}}} - f_{0} \right\|_{L_{2}}^2  \psi_{n}(\alpha_{i})^{-1} \mathbf{1}_{\mcE_{j}} \lesssim s^{2}
\end{equation*}
with probability at least $1 -4\exp\{ -s (\logn)^{\frac{\alpha_{i}}{2\alpha_{i} + d}} \}$.

Now, we turn to $j < i$. By the definition of $\mcE_{j}$, there exists $\alpha'\in \mcA$ with $\alpha' < \alpha_{i}$ such that $\|\hat{f}_{\lambdat_{\alpha_{i}}} - \hat{f}_{\lambdat_{\alpha'}}\|_{L_{2}}^2 \psi_{n}(\alpha')^{-1} > c_{0}^2$. Hence, by triangle inequality, we have either $\|\hat{f}_{\lambdat_{\alpha_{i}}} -f_{0}\|_{L_{2}}^2 \psi_{n}(\alpha')^{-1} > c_{0}^2/4$ or $\| \hat{f}_{\lambdat_{\alpha'}} - f_{0}\|_{L_{2}}^2 \psi_{n}(\alpha')^{-1} > c_{0}^2/4$. Then, we have 
\begin{equation}\label{eqn: prob of event E_j}
    \mbP \left(\mcE_{j}\right)  \leq \sum_{\ell=1}^{i-1} \left( 
    \mbP\left( \|\hat{f}_{\lambdat_{\alpha_{i}}} -f_{0}\|_{L_{2}}^2 \psi_{n}(\alpha_{\ell})^{-1} > c_{0}^2/4 \right) + \mbP\left(\| \hat{f}_{\lambdat_{\alpha_{\ell}}} - f_{0}\|_{L_{2}}^2 \psi_{n}(\alpha_{\ell})^{-1} > c_{0}^2/4 \right)  \right).
\end{equation}
Since for $\ell < i$, $f_{0}\in H^{\alpha_{i}} \subset H^{\alpha_{\ell}}$ leads to 
\begin{equation*}
    \mbP\left( \|\hat{f}_{\lambdat_{\alpha_{i}}} -f_{0}\|_{L_{2}}^2 \psi_{n}(\alpha_{\ell})^{-1} > c_{0}^2/4 \right) \leq \mbP\left( \|\hat{f}_{\lambdat_{\alpha_{\ell}}} -f_{0}\|_{L_{2}}^2 \psi_{n}(\alpha_{\ell})^{-1} > c_{0}^2/4 \right),
\end{equation*}
it is sufficient to focus on the concentration of $\hat{f}_{\lambdat_{\ell}}$ to $f_{0}$. Setting $s = c_{0}/2$ in (\ref{eqn: prob inequality with logn}), we have 
\begin{equation*}
    \mbP\left( \|\hat{f}_{\lambdat_{\alpha_{\ell}}} -f_{0}\|_{L_{2}}^2 \psi_{n}(\alpha_{\ell})^{-1} > c_{0}^2/4 \right) \leq 4\exp\left\{ - \frac{c_{0}}{2} (\logn)^{\frac{\alpha_{\ell}}{2\alpha_{\ell} + d}}\right\}.
\end{equation*}
Thus,
\begin{equation*}
    \mbP(\mcE_{j}) \leq 8 \exp\left\{ - \frac{c_{0}}{2} (\logn)^{\frac{\alpha_{\min}}{2\alpha_{\min} + d}} \right\}.
\end{equation*}
Combining the above result with (\ref{eqn: prob of event E_j}) and noting that on $\mcE_{j}^{c}$, $\mathbf{1}_{\mcE_{j}} = 0$, it yields that 
\begin{equation*}
    \sum_{j=1}^{i-1} \left\| \hat{f}_{\lambdat_{\alpha_{j}}} - f_{0} \right\|_{L_{2}}^2  \psi_{n}(\alpha_{i})^{-1} \mathbf{1}_{\mcE_{j}} \lesssim s^{2}
\end{equation*}
with probability at least $1 - 4\exp\{ -s (\logn)^{\frac{\alpha_{i}}{2\alpha_{i} + d}} \} - 4\exp\{ -\frac{c_{0}}{2} (\logn)^{\frac{\alpha_{\min}}{2\alpha_{\min} + d}} \} (\logn)^2$. 

Finally, the proof is finished by combining the results for summation over $j\geq i$ and summation over $j< i$. 
\end{proof}

\section{Additional Simulation Results}\label{apd: simulation}
Reproducible codes for generating all simulation results in the main paper and the following sections are available in the supplemental material.

\subsection{Additional Results for Target-Only KRR with Gaussian kernels}
In Section~\ref{subsec: simulation for Adaptivity of Gaussian KRR}, we only present the best lines with the optimal $C$. In this part, we report the generalization error decay curve for different $C$. We report the results in Figure~\ref{fig: adaptive and nonadaptive rate for target-only KRR with all slope}. Each subfigure contains 7 different lines that center around the optimal line. One can see even with different $C$, the empirical error decay curves are still aligned with the theoretical ones.

\begin{figure}[ht]
    \centering
    \includegraphics[width = \textwidth]{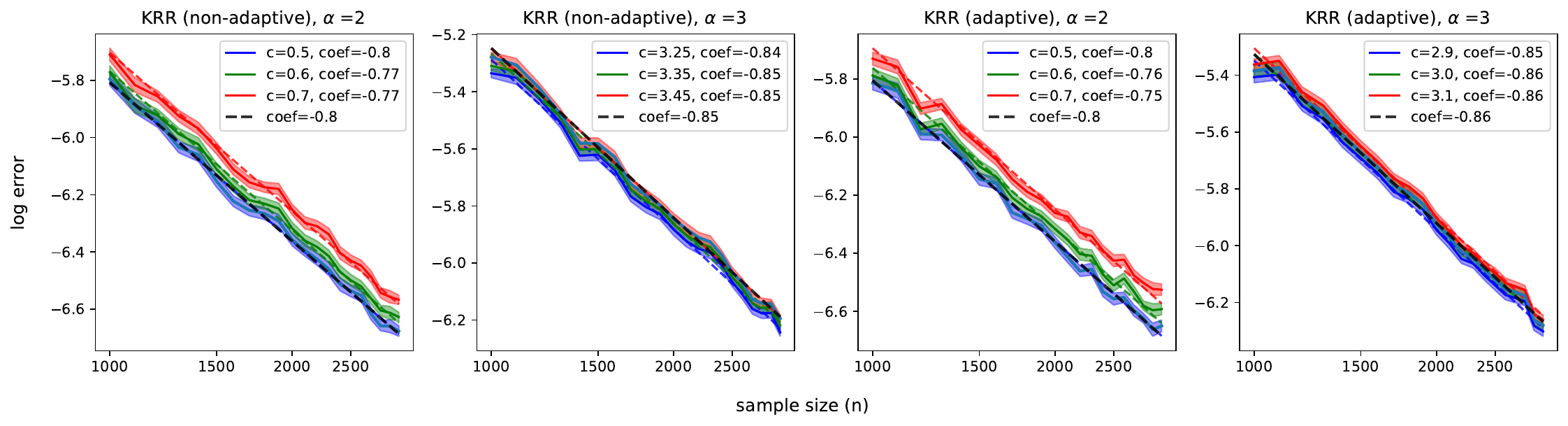}
    \caption{Error decay curves of target-only KRR based on Gaussian kernel, both axes are in log scale. The curves with different colors correspond to different $C$ and denote the average logarithmic generalization errors over 100 trials. The dashed black lines denote the theoretical decay rates.
    }
    \label{fig: adaptive and nonadaptive rate for target-only KRR with all slope}
\end{figure}

\subsection{Additional Results for TL Algorithm Comparison}

\paragraph{Implementation of Finite Basis Expansion:}

Denote the finite basis estimator (FBE) for a regression function as 
\begin{equation*}
    \hat{f}_{M}(x) = \sum_{j=1}^{M} \beta_{j} B_{j}(x)
\end{equation*}
where $B_{j}$ are given finite basis or spline functions with a different order and $M$ denotes the truncation number which generally controls the variance-bias trade-off of the estimator. Then the transfer learning procedure proposed in \citet{wang2016nonparametric} can be summarized by the following $4$ steps.
\begin{enumerate}
    \item Estimate $f_{S}$ using the FBE and source data, output $\hat{f}_{S,M_{1}}$
    \item Produce the pseudo label $\hat{y}_{T,i}$ using $\hat{f}_{S,M_{1}}$ and $x_{T,i}$, obtain the offset estimation as $y_{T,i} - \hat{y}_{T,i}$.
    \item Estimate the offset function using the FBE with $\{(y_{T,i} - \hat{y}_{T,i}, x_{T,i})\}_{i=1}^{N_{T}}$, output $\hat{f}_{\delta,M_{2}}$.
    \item Return $\hat{f}_{S,M_{1}} + \hat{f}_{\delta,M_{2}}$ as the estimator for $f_{T}$.
\end{enumerate}

\paragraph{Regularization Selection in SATL:}
In target-only KRR results, for all $n$, we fixed the constant $C$ and reported the best generalization error decay curves in Figure~\ref{fig: adaptive and nonadaptive rate for target-only KRR with the best slope} and other error decay curves for other $C$s in Figure~\ref{fig: adaptive and nonadaptive rate for target-only KRR with all slope}. In SATL, one can also conduct a similar tuning strategy and select the best performer $C$. However, this can be computationally not sufficient. For example, if one has $40$ candidate for $C$, than there would be total $40^2$ constants combinations in two-step transfer learning process. Such a problem also appeared in FBE approaches where one needs to tune the optimal $M$ (number of basis or the degree of Bspline). 

Therefore, for each $\alpha$, we determine the constant $C$ in $\exp\{-Cn^{-\frac{1}{2\alpha + d}}\}$ via following cross-validation (CV) approach.
We consider the largest sample size in the current setting, i.e. largest $n_{S}$ while estimating the source model and largest $n_{T}$ while estimating the offset, and the estimate $C$ is obtained by the classical K-fold CV, then the estimated $\hat{C}$ is used for all sample size in the experiments.

\paragraph{Additional Results for different basis in FBE:}
When comparing our SATL with the TL algorithm proposed in \citet{wang2016nonparametric}, we only present the result of using B-spline in the main paper. In this part, we give a detailed description of the implementation for the comparison and provide additional simulation results based on other basis.

\begin{figure}[ht]
    \centering
    \includegraphics[width = 0.8\textwidth]{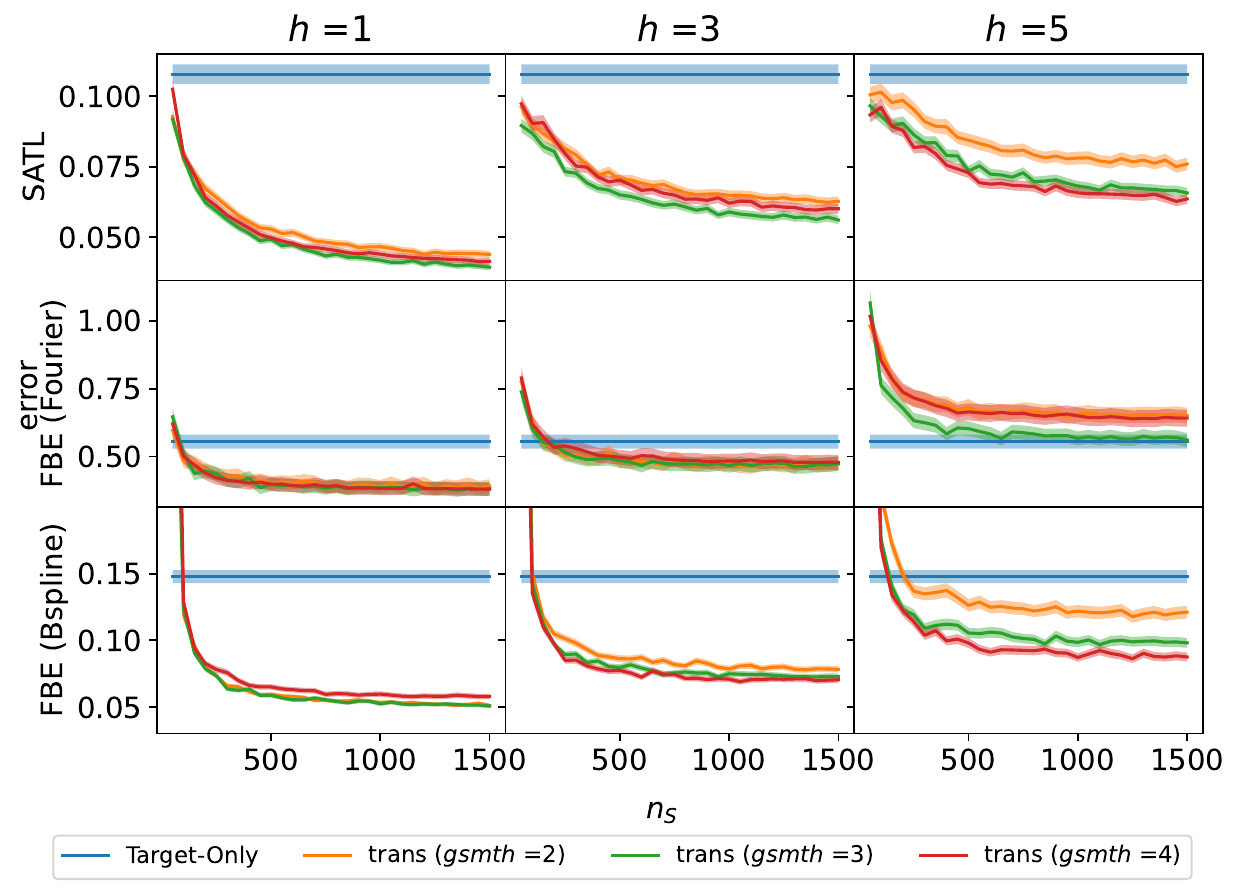}
    \caption{Generalized error for fixed $n_{T}$ scenario under different $h$ and smoothness of $f_{T}-f_{S}$. Each row represents SATL, FBE-based-TL with Fourier basis and B-spline respectively.}
    \label{fig: KRR, BE, Bspline for fixed target size}
\end{figure}

In our implementation, we consider the finite basis as (1) Fourier basis $B_{j}(x) = \sqrt{2}cos(\pi *k*x)$ (which was used in \citet{wang2016nonparametric}) and (2) $B_{j}$ being the $j$-th order B-spline. In \citet{wang2016nonparametric}, the authors use $m_{1} = m_{2}=500$ but we notice this will hugely degrade the algorithm performance. Therefore, we use CV to select $m_{1}$ and $m_{2}$ to optimize the algorithm performance. We now present the generalization error for all three candidates under the fixed $n_{T}$ scenario in Figure~\ref{fig: KRR, BE, Bspline for fixed target size}. Focusing on the Fourier Basis row, we can see the error is much higher than our proposed GauKK-TL and B-spline approach. Such a performance is expected as the fixed basis one employed doesn't match the underlying structure of the source and target function.

\end{document}